\documentclass[journal]{IEEEtran}
\IEEEoverridecommandlockouts
\usepackage{cite}
\usepackage{amsmath,amssymb,amsfonts,amsthm}
\usepackage{enumitem}
\usepackage{algorithm, algpseudocode}
\usepackage{booktabs, multirow}
\newtheoremstyle{italic_lemma}
  {0pt}   
  {0pt}   
  {\itshape}  
  {\parindent} 
  {\itshape}  
  {:}         
  {.5em}      
  {\itshape \thmname{#1}\thmnumber{ #2}\thmnote{ (#3)}} 
\theoremstyle{italic_lemma}
\newtheorem{lemma}{Lemma}
\newtheorem{theorem}{Theorem}
\newtheorem{assumption}{Assumption}
\newtheorem{corollary}{Corollary}

\usepackage{graphicx}
\usepackage{subcaption}
\usepackage{textcomp}
\usepackage{xcolor}
\def\BibTeX{{\rm B\kern-.05em{\sc i\kern-.025em b}\kern-.08em
    T\kern-.1667em\lower.7ex\hbox{E}\kern-.125emX}}

\usepackage{array}
\usepackage{stfloats}
\usepackage{url}
\usepackage{verbatim}
\usepackage{balance}
\usepackage{makecell}
\usepackage[hidelinks]{hyperref}

\newcommand{\norm}[1]{\left\lVert#1\right\rVert}
\newcommand{\bs}{\boldsymbol}

\def\wkc{\bs{w}_{k,\mathrm{c}}}
\def\wks{\bs{w}_{k,\mathrm{s}}}
\def\wc{\bs{w}_{\mathrm{c}}}
\def\ws{\bs{w}_{\mathrm{s}}}
\def\dc{d_\mathrm{c}}
\def\ds{d_\mathrm{s}}
\def\qc{q_\mathrm{c}}
\def\qs{q_\mathrm{s}}
\def\qu{q_\mathrm{u}}
\def\qmax{q_{\max}}
\def\Qc{Q_\mathrm{c}}
\def\Qs{Q_\mathrm{s}}
\def\lc{\ell_\mathrm{c}}

\def\pkc{p_{k,\mathrm{c}}}
\def\ps{p_{\mathrm{s}}}
\def\Skc{S_{k,\mathrm{c}}}
\def\Emac{E_{\mathrm{MAC}}}
\def\El{E_{\mathrm{l}}}
\def\Em{E_{\mathrm{m}}}
\def\Ed{E_{\mathrm{d}}}
\def\Mc{M_{\mathrm{c}}}
\def\Oc{O_{\mathrm{c}}}
\def\Ms{M_{\mathrm{s}}}
\def\Os{O_{\mathrm{s}}}
\def\Efpkc{E^{\mathrm{FP}}_{k,\mathrm{c}}}
\def\Ebpkc{E^{\mathrm{BP}}_{k,\mathrm{c}}}
\def\Eckc{E^{\mathrm{C}}_{k,\mathrm{c}}}
\def\Eakc{E^{\mathrm{A}}_{k,\mathrm{c}}}
\def\Ewkc{E^{\mathrm{W}}_{k,\mathrm{c}}}
\def\Edramkc{E^{\mathrm{DRAM}}_{k,\mathrm{c}}}
\def\Efpks{E^{\mathrm{FP}}_{k,\mathrm{s}}}
\def\Ebpks{E^{\mathrm{BP}}_{k,\mathrm{s}}}
\def\Edramks{E^{\mathrm{DRAM}}_{k,\mathrm{s}}}
\def\rku{r_k^\mathrm{u}}
\def\rkd{r_k^\mathrm{d}}
\def\rb{r^\mathrm{b}}

\def\Eku{E_k^\mathrm{U}}
\def\Eb{E^\mathrm{B}}
\def\tka{\tau_k^\mathrm{A}}
\def\tkg{\tau_k^\mathrm{G}}
\def\tku{\tau_k^\mathrm{U}}
\def\tb{\tau^\mathrm{B}}
\def\tfpkc{\tau_{k,\mathrm{c}}^\mathrm{FP}}
\def\tfps{\tau_{\mathrm{s}}^\mathrm{FP}}
\def\tbpkc{\tau_{k,\mathrm{c}}^\mathrm{BP}}
\def\tbps{\tau_{\mathrm{s}}^\mathrm{BP}}
\def\wkt{\bs{w}_k^t}
\def\wkto{\bs{w}_k^{t+1}}
\def\vkt{\bs{v}_k^{t}}
\def\vkto{\bs{v}_k^{t+1}}
\def\wktq{\bs{\tilde{w}}_{k}^{t}}

\def\tp{\mathsf{T}}
\def\itp{\mathsf{\scriptscriptstyle T}}

\newcommand{\sqb}[1]{\left[#1\right]}
\newcommand{\paren}[1]{\left(#1\right)}
\newcommand{\curly}[1]{\left\{#1\right\}}
\newcommand{\floor}[1]{\lfloor #1 \rfloor}
\newcommand{\expec}[1]{\mathbb{E}\sqb{#1}}
\newcommand{\iprod}[2]{\left\langle#1\,,\,#2\right\rangle}

\let\originalleft\left
\let\originalright\right
\renewcommand{\left}{\mathopen{}\mathclose\bgroup\originalleft}
\renewcommand{\right}{\aftergroup\egroup\originalright}

\begin{document}

\title{GQ-FSL: Green Quantized Federated Split Learning Framework for Wireless Edge Networks}

\author{\IEEEauthorblockN{Idan Roth and Lutz Lampe}\\
\IEEEauthorblockA{Department of Electrical and Computer Engineering, \\The University of British Columbia,  Vancouver, BC, Canada, \{idanroth, Lampe\}@ece.ubc.ca}
\thanks{This work is supported by the Natural Sciences and Engineering Research Council of Canada (NSERC) and the NATO Science for Peace and Security Programme (G7646). Part of this work has been accepted for presentation at the 2026 IEEE International Workshop on Signal Processing and Artificial Intelligence in Wireless Communications (SPAWC)~\cite{gqfsl_spawc2026}. This work has been submitted to the IEEE for possible publication. Copyright may be transferred without notice, after which this version may no longer be accessible.}
}

\maketitle

\begin{abstract}
Deploying state-of-the-art deep neural networks (DNNs) at the wireless edge is severely bottlenecked by the strict energy and resource constraints of mobile devices. Although federated split learning (FSL) alleviates on-device computational burdens by offloading workloads to an edge server, this may introduce systemic overheads, while the continuous exchange of intermediate activations, gradients, and submodels still incurs significant energy consumption (EC). To address this, we propose a green quantized FSL (GQ-FSL) framework that incorporates stochastic quantization for both local collaborative training and wireless transmissions. Notably, GQ-FSL supports asymmetric precision levels for the client- and server-side submodels, effectively decoupling device energy constraints from global convergence degradation. To quantify these tradeoffs, we develop parameterized energy models for the split architecture and derive a theoretical convergence bound under statistically heterogeneous data. Building on that, we formulate a joint optimization problem to configure the DNN split point and precision levels, minimizing the total system EC while satisfying strict latency and target accuracy constraints. Ultimately, we demonstrate that GQ-FSL enables large-scale DNN deployment on resource-constrained devices, achieving superior energy efficiency compared to quantized federated learning and full-precision FSL. 
\end{abstract}

\vspace{-1.15mm}

\section{Introduction}

 \IEEEPARstart{W}{ith} the ubiquitous deployment of Internet of Things (IoT) and mobile devices, their increasing sensing and computational capabilities generate massive volumes of data that serve as the foundation for data-driven artificial intelligence (AI) models \cite{zhu2020toward_int_edge,letaief2021edge_ai_6g}. Conventional centralized learning entails gathering and processing this raw data in the cloud. However, latency and backhaul bandwidth limitations have driven a paradigm shift towards edge computing. 
 By integrating cloud-like functionalities within the radio access network (RAN), multi-access edge computing (MEC) \cite{Porambage2018_mec} enables a close synergy with AI, giving rise to the edge AI paradigm. This paradigm pushes inference and training processes towards the network edge in close proximity to data sources, transforming wireless networks from ``connected things'' to ``connected intelligence'' \cite{zhou2019edge_paving, shi2020communication_eff_edgeai,letaief2021edge_ai_6g}. 
Nevertheless, the widespread adoption of AI-capable devices and the increasing complexity of learning tasks will significantly increase energy consumption (EC) within the RAN. 
Consequently, addressing the energy cost of edge AI workloads is a fundamental imperative for the sustainability of next-generation wireless networks \cite{schwartz2020green_ai, mao2024green_edge, thakur2025green_fl}. 


As a cornerstone of edge AI, federated learning (FL) \cite{mcmahan2017communication, lim2020edgefederated, niknam2020wirelessfederated} has become a fundamental framework for distributed model training. It enables edge devices to collaboratively train a global model using locally stored data, sharing only model updates with a nearby server to preserve privacy.
However, state-of-the-art AI models typically employ deep neural networks (DNNs) requiring tens of millions of parameters and billions of multiply-accumulate (MAC) operations \cite{schwartz2020green_ai,kim2023greenfl,lin2024efficient_psl}. Since FL relies exclusively on full on-device training, deploying such complex models imposes an immense computational burden and severe memory requirements. Furthermore, its iterative nature demands frequent communication rounds to synchronize these massive models over the wireless channel. Given the resource-constrained nature of IoT and mobile devices, accommodating these intense overheads often renders standard FL impractical \cite{lin2024efficient_psl, lin2024split_learn_6G}. 
To alleviate these on-device burdens, split learning (SL) \cite{vepakomma2018split_learning} partitions a DNN at a specific cut-layer into a lightweight client-side submodel and a computationally intensive server-side submodel. The edge client executes only the initial layers and transmits intermediate activations (and labels) to the server, which completes the forward propagation (FP) and initiates backpropagation (BP). The server then returns the cut-layer gradients to the client to complete the local update.
While SL reduces on-device complexity, its standard, sequential client-by-client execution yields excessive training latency \cite{lin2024split_learn_6G, thapa2022splitfed}.
Federated split learning (FSL) \cite{thapa2022splitfed} overcomes this limitation by enabling concurrent, parallel training across multiple devices, establishing the split-based framework as a robust alternative to FL \cite{lin2024efficient_psl, lin2024split_learn_6G}. Nevertheless, FSL retains significant system overheads. 
A resource-limited edge server can be quickly overwhelmed by the aggregated client workloads \cite{lin2024efficient_psl}, and transferring cut-layer data may introduce a large volume of wireless exchange \cite{shao2021communication_computation_tradeoff}. Furthermore, a potentially heavy client-side submodel still imposes intense local processing demands, a substantial memory footprint, and periodic model synchronization overheads. Altogether, the resulting high EC and latency costs pose a fundamental challenge in deploying FSL at the wireless edge while striving to achieve green, sustainable edge AI \cite{thakur2025green_fl}.

Model compression techniques have attracted significant attention for facilitating efficient deep learning \cite{cheng2017survey_compression, mao2024green_edge, thakur2025green_fl}. As a prominent compression strategy, model quantization offers a direct mechanism to mitigate the aforementioned overheads and improve the overall energy footprint, thereby serving as a key enabler for deploying green FSL frameworks on resource-constrained edge devices.
By reducing the numerical precision of DNN parameters and activations to lower-bit formats \cite{hubara2018quantized_nn}, quantization simultaneously lowers local computational complexity, reduces memory requirements, and shrinks the communication payload. However, low-precision representations introduce quantization errors that degrade the convergence rate. Consequently, identifying an appropriate precision level that balances energy efficiency and convergence performance remains a key challenge \cite{kim2023greenfl}. Notably, this balance must explicitly account for the non-trivial tradeoff between EC and training latency, as quantization inherently alters both the per-round execution time and the total number of required rounds for convergence.

\subsection{Related Works}
Many research efforts have optimized wireless edge networks for energy-efficient FL deployment. Such approaches, for example, improve standard FL efficiency by jointly managing communication and computation resources, while characterizing the tradeoff between EC and training latency \cite{tran2019federated_tradeoff, yang2020energy_efficient_fl, luo2021cost_effective_fl, zeng2021energy_efficient_gpu_cpu}, analyzing carbon footprints \cite{savazzi2022energy_carbon_footprint}, and deploying client-edge-cloud hierarchical architectures \cite{liu2020client_edge_cloud}. While effective, these approaches fundamentally rely on the assumption that edge devices possess the memory and computational capacities to train the entire model locally.
Recent studies have investigated efficient designs specifically for split-based architectures. 
By eliminating client-side model exchanges, parallel SL frameworks balance EC and training time through optimized cut-layer selection \cite{kim2023bargaining_psl}, or minimize latency by shrinking overheads via last-layer gradient aggregation alongside the joint optimization of subchannels, power, and cut-layer selection \cite{lin2024efficient_psl}. To address device heterogeneity in FSL, \cite{xu2023accelerating_fsl} optimizes individual split points and bandwidth allocation to minimize latency, while \cite{zhu2024esfl} balances client-side workloads and server computing resources. Targeting the tradeoff between EC and training latency in FSL, \cite{samikwa2022ares} adapts device-specific splitting points based on time-varying network and computing conditions, an approach later expanded by \cite{samikwa2024dfl} through dynamic device clustering to handle data heterogeneity. Additionally, \cite{ao2024federated_FSL_wanabe} jointly optimizes power allocation, device scheduling, and cut-layer selection to minimize a weighted sum of per-round latency and EC under heterogeneous communication and computing constraints. 
However, despite system-level efficiency and computation offloading, these frameworks fail to resolve the fundamental bottlenecks discussed earlier. The processing and communication of full-precision parameters and cut-layer data continue to severely exacerbate EC and training latency. Crucially, by treating the DNN's numerical precision as a fixed constraint, these approaches overlook the substantial energy savings achievable through lower-bit representations.

To alleviate the stringent constraints of edge AI, model quantization has primarily been explored to overcome communication overheads. For instance, \cite{zhu2020one_bit_ota} analyzes the convergence of a one-bit gradient quantization scheme for over-the-air FL, while the FL framework in \cite{reisizadeh2020fedpaq} combines uploading quantized model updates with periodic averaging and partial device participation. To minimize training time, \cite{liu2022training_grad_quant} jointly optimizes gradient quantization levels and bandwidth allocation. Meanwhile, \cite{feng2021design_quant_device} balances communication and computation costs through optimized uplink model compression design and device selection strategies. In SL, \cite{oh2025communication_dropout_quant} compresses intermediate activations using quantization and adaptive dropout, although its reliance on sequential training induces significant latency. Ultimately, these works prioritize communication efficiency, largely overlooking the substantial storage requirements, computational burden, and EC incurred by executing high-precision local training on resource-constrained devices.

A limited number of studies have extended quantization to simultaneously address energy-efficient communication, storage, and computation during training. For example, \cite{yang2021communication_bnn_fl} employs binary neural networks in FL to reduce computation and communication costs, but lacks a mechanism to optimize precision levels for balancing model accuracy and EC. Addressing this, \cite{chen2022energy_qfl} minimizes energy by jointly optimizing the number of local iterations, bandwidth, and quantization precision. However, it enforces identical precision levels for local training and transmission and relies on full device participation, making it vulnerable to stragglers. Building on this, \cite{kim2023greenfl} develops a quantized FL framework utilizing quantized neural networks (QNNs) with distinct precision levels for computation and transmission. 
While they analyze the tradeoff between energy efficiency and convergence rate, optimizing purely for convergence rounds fails to capture the true physical training latency. 
More fundamentally, all these frameworks strictly rely on full on-device training, imposing considerable computational and memory burdens that are often infeasible for resource-constrained edge devices. As a result, while quantization has proven effective in standard FL, its integration into FSL remains largely unexplored, particularly regarding the interplay between quantization precision levels and splitting point selection, alongside their collective impact on the energy-latency tradeoff and convergence behavior.

\subsection{Contributions and Organization}
Motivated by these open challenges, this paper investigates a QNN-based FSL framework designed for energy-efficient collaborative training on resource-constrained devices. The major contributions of this paper are summarized as follows: 

\begin{itemize}[leftmargin=5.8mm]
    \item We introduce a green quantized FSL (GQ-FSL) framework that partitions a DNN into client- and server-side submodels and integrates stochastic quantization into both local training and wireless transmissions. By supporting distinct precision levels for client and server QNNs alongside the client-side submodel upload, GQ-FSL substantially reduces computational load, memory accesses, and communication overhead to minimize total EC.

    \item To accurately quantify the deployment costs, we derive split-based parameterized models for the EC and latency of QNN-based local training on a generic DNN processing platform, alongside comprehensive models for wireless uplink and downlink transmissions. 
    
    \item We establish a rigorous theoretical foundation by deriving a convergence bound that explicitly captures the impact of stochastic quantization under statistically heterogeneous data distributions. This formally quantifies the fundamental tradeoff between energy savings and quantization-induced convergence degradation. Our derivation refines prior QNN-based FL analysis \cite{kim2023greenfl} to ensure a strictly valid and tighter error bound, providing a unified theoretical guarantee for both standard quantized FL and FSL.

    \item Building on these derivations, we formulate a joint optimization problem to strategically configure the splitting point, precision levels, client participation, and local iterations. This formulation minimizes the total system EC subject to target-accuracy and latency constraints. Our analysis reveals fundamental design insights regarding the inherent tradeoff between energy efficiency and training latency. Notably, we demonstrate a highly advantageous dynamic: theoretical convergence deceleration does not strictly dictate the physical training time. While reduced precision levels and limited client participation degrade the convergence rate to achieve substantial energy savings, the simultaneous accelerated per-round computations and transmissions actually compress the overall training latency up to a threshold of diminishing returns. 
    
    \item Simulation results confirm that GQ-FSL facilitates the deployment of large-scale DNNs on resource-constrained devices, while achieving superior energy efficiency and latency management compared to standard quantized FL and full-precision FSL architectures.

\end{itemize}

The remainder of this paper is organized as follows. Section~\ref{sec:framework} elaborates on the proposed GQ-FSL framework, while Section~\ref{sec:energy_latency} develops the associated energy and latency models. In Section~\ref{sec:problem_formulation}, we formulate the energy minimization problem of the framework, analyze the convergence behavior, and outline the solution approach. Section~\ref{sec:results} discusses the simulation results, and Section~\ref{sec:conclusion} provides concluding remarks.

\section{Proposed GQ-FSL Framework} \label{sec:framework}
In this section, we formulate the proposed GQ-FSL framework. We first establish the standard FSL system architecture and its global objective. Next, we introduce the mathematical foundations of the stochastic quantization scheme. Building upon these paradigms, we then detail our novel QNN-based FSL joint architecture. By leveraging decoupled precision levels across the partitioned network, this joint approach provides the architectural flexibility to mitigate device constraints without severely degrading global convergence. Finally, we outline the iterative collaborative learning procedure executed across the wireless edge.

\subsection{FSL System Model}

\begin{figure}[t]
     \centering
     \includegraphics[width=\linewidth]{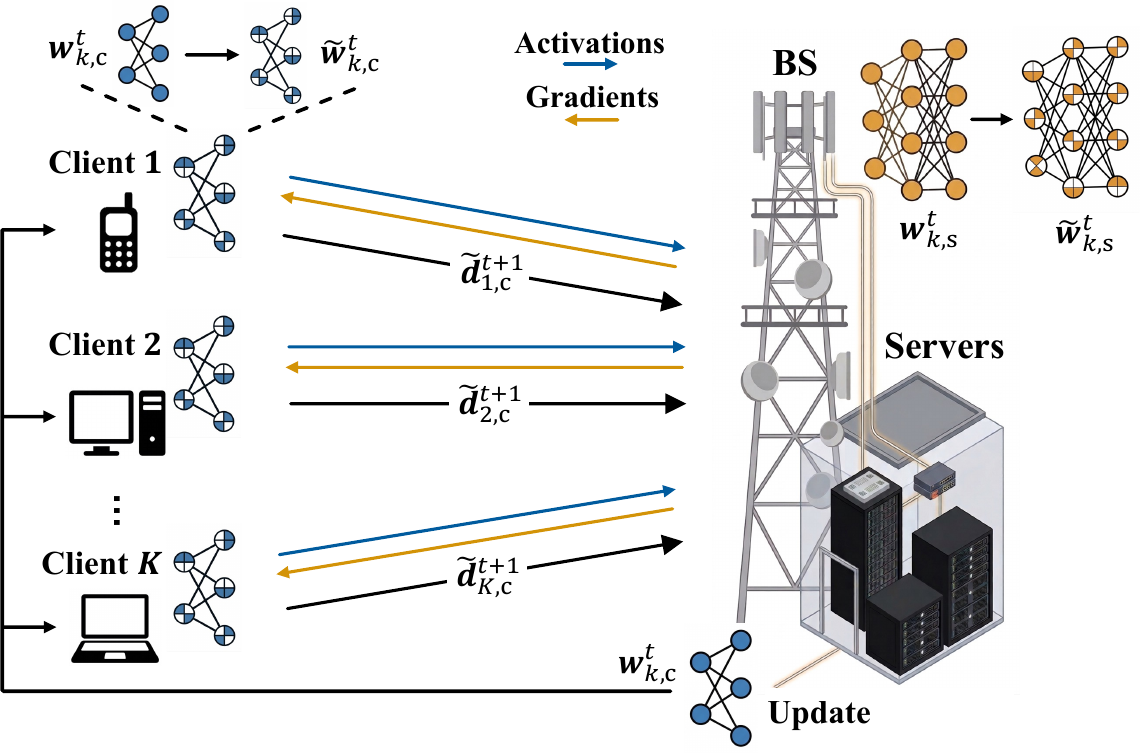}
     \caption{Illustration of GQ-FSL over wireless network.}
     \label{fig:system}
\end{figure}

We consider an FSL system deployed over a wireless edge network as shown in Fig.~\ref{fig:system}. The network consists of a base station (BS) that serves a set of $N$ resource-constrained clients (devices), denoted by $\mathcal{N} = \{1, 2, \dots, N\}$. Each client $k \in \mathcal{N}$ holds a local dataset $\mathcal{D}_k = \{\bs x_{kj},  y_{kj}\}$, where $j = 1, \dots, D_k$, and $\{\bs x_{kj}, y_{kj}\}$ is an input-output pair of a classification problem ($\bs x_{kj}$ being the input vector while $y_{kj}$ is the corresponding label). Therefore, the entire dataset $\mathcal{D}=\bigcup_{k=1}^N\mathcal{D}_k$ is of size $D=\sum_{k=1}^N D_k$. The datasets across devices are statistically heterogeneous, commonly referred to as non-independent and identically distributed (non-IID).
The BS is co-located with two distinct computational entities: a powerful edge server responsible for coordinating the split learning process, and a separate, non-colluding parameter server dedicated to model aggregation. Unlike FL, where each client independently trains the entire model $\bs {w}_k\in\mathbb{R}^d$ locally, FSL partitions the model into a client-side submodel $\wkc\in\mathbb{R}^{\dc}$ and a server-side submodel $\wks\in\mathbb{R}^{\ds}$. Here, $\dc$ and $\ds$ represent the number of parameters in the client- and server-side submodels, respectively, such that $d=\dc+\ds$ denotes the total number of parameters of the complete model $\bs{w}_k$.

We define a loss function $f\paren{\bs x_{kj}, y_{kj} \,;\bs w}$ to quantify the performance of a DNN with parameters $\bs w$ over the local data. Therefore, the local loss function for each client $k$ over its dataset is given by $F_k(\bs w) = \frac{1}{D_k}\sum_{j=1}^{D_k} f\paren{\bs x_{kj}, y_{kj} \,;\bs w}$. The global FSL objective is to find the global parameters $\bs{w}=\big[\wc^\itp, \ws^\itp\big]^\itp$ that minimize the total loss: 
\begin{equation}
    \min_{\bs w} F\paren{\bs w} = \min_{\bs w}\sum_{k=1}^N \frac{D_k}{D} F_k\paren{\bs w},
    \label{eq:fsl_opt}
\end{equation}
which enforces global consensus $\bs w_k = \bs w$ across all local models $\bs{w}_k=\big[\wkc^\itp, \wks^\itp\big]^\itp$, $k\in\mathcal{N}$.

Solving problem \eqref{eq:fsl_opt} typically requires an iterative process between the BS and the clients. In practical wireless networks and IoT systems, however, edge devices are constrained by limited computing and energy resources. While FSL alleviates the burden of executing the entire model locally, training DNNs with full-precision parameters still incurs substantial computational complexity and memory usage.




\subsection{Quantization Scheme}
\label{sec:proposed_qnn}

To overcome these bottlenecks, our framework employs model quantization to systematically reduce and manage the numerical precision of the DNN. Specifically, we represent the model parameters in a fixed-point format $[\Omega. \Phi]$, where $\Omega$ and $\Phi$ denote the integer and fractional parts, respectively \cite{gupta2015deep_fixedpoint}. By allocating one bit for the integer part (sign bit) and $\paren{q-1}$ bits for the fractional part, the quantization step size is $\kappa = 2^{-\paren{q-1}}$, yielding a dynamic range of $\sqb{-1, 1-\kappa}$. As standard in QNNs, weights are bounded within $\sqb{-1,1}$ to prevent excessive growth that offers no meaningful gain in quantized representation \cite{hubara2018quantized_nn}. 

We adopt a stochastic quantization scheme, which generally exhibits superior performance compared to deterministic quantization during training \cite{gupta2015deep_fixedpoint}. Any given continuous value $w\in\mathbb{R}$ can be stochastically quantized to a precision level $q$ as follows:
\begin{equation}
Q(w) = \begin{cases} \floor{w}, &    \text{with probability }\ \frac{\floor{w} + \kappa - w}{\kappa} \\ 
\floor{w} + \kappa, & \text{with probability }\ \frac{w - \floor{w}}{\kappa} 
\end{cases}, \label{eq:stoc_quan}
\end{equation}
where $\floor{w}$ is the largest integer multiple of $\kappa$ that is less than or equal to $w$. As in~\cite{gupta2015deep_fixedpoint}, $w$ values that exceed the dynamic range $\sqb{-1, 1-\kappa}$ are saturated to either the lower or the upper bound. We generalize the stochastic quantization model to a vector $\bs{w}=\sqb{w_1,...,w_d}^\itp \in \mathbb{R}^d$. Each element $w_i$ can be quantized independently with a distinct precision level $q_i$. Let $Q_{i}(\cdot)$ denote the stochastic quantization function in \eqref{eq:stoc_quan}, which is applied on $w_i$ with precision level $q_i$ and the corresponding $\kappa_i$.
Building upon the scalar quantization properties established in \cite{kim2023greenfl}, the mathematical properties of this stochastic quantization are given in the following lemma.

\begin{lemma} \label{lem:quatization}
    For a vector $\bs{w}=\sqb{w_1,...,w_d}^\itp\in\mathbb{R}^d$, and an element-wise vector stochastic quantization $Q(\bs{w}) = \sqb{Q_{1}(w_1), \dots, Q_{i}(w_i), \dots, Q_{d}(w_d)}^\itp$, we have
\end{lemma} \vspace{-4.5mm}
\begin{align}
    \mathbb{E}\sqb{Q_{i}(w_i)} = w_i,& \quad \mathbb{E}\sqb{(Q_{i}(w_i)-w_i)^2} \le \frac{1}{2^{2q_i}}, \label{eq:lemma1_scalar}\\
    \mathbb{E}\sqb{Q(\bs{w})} = \bs{w},& \quad \mathbb{E}\sqb{\norm{Q(\bs{w})-\bs{w}}^2} \le \sum_{i=1}^d \frac{1}{2^{2q_i}} \label{eq:lemma1_vector}.
\end{align}

\begin{IEEEproof}
    The scalar unbiasedness and variance bound in~\eqref{eq:lemma1_scalar} follow directly from the properties of stochastic quantization as proven in \cite{kim2023greenfl}. The vector unbiasedness in~\eqref{eq:lemma1_vector} 
    follows from its scalar counterpart. By using the scalar variance bound in~\eqref{eq:lemma1_scalar}, we obtain the generalized bound for the vector $\bs{w}$:
\begin{equation*}
    \mathbb{E}\sqb{\norm{Q(\bs{w})-\bs{w}}^2} \!=\! \sum_{i=1}^d \mathbb{E}\sqb{(Q_{i}(w_i)-w_i)^2} \le \sum_{i=1}^d \frac{1}{2^{2q_i}}. 
    \vspace{-1.5mm}
\end{equation*}
\end{IEEEproof}
 


\subsection{The Joint GQ-FSL Architecture} \label{sec:proposed_gqfsl}
Our GQ-FSL framework integrates the stochastic quantization scheme into the local training and wireless transmission processes of standard FSL. We employ a QNN architecture where weights and activations are quantized and represented in a fixed-point format rather than the conventional 32-bit floating-point format \cite{hubara2018quantized_nn, kim2023greenfl}. To further alleviate the computational and memory burdens on edge devices, the QNN is partitioned into two segments (submodels) at the cut layer, denoted by $\lc$, which serves as the network's splitting point.

The primary objective of the framework is to collaboratively train a partitioned QNN of a homogeneous architecture across all clients. In our proposed framework, we permit distinct quantization precision levels for the different segments of the QNN. The client-side submodel is quantized to a precision level of $\qc$ bits, while the server-side submodel operates at a precision level of $\qs$ bits. Furthermore, we apply a distinct precision level of $\qu$ bits specifically for quantizing the client-side submodel updates prior to uplink transmission for global aggregation. 
The splitting point $\lc$ and the precision levels $\qc$, $\qs$, and $\qu$ are kept identical across all clients and are fixed throughout the training process. Although higher accuracy can be achieved by increasing the precision levels, it comes at a cost of higher EC. This joint QNN-based FSL approach can significantly reduce the total EC associated with local computations, memory accesses, and wireless transmissions. Furthermore, it enables the deployment of large-scale DNNs on resource-constrained edge clients.

\subsubsection{Quantization Error Analysis}
From Lemma~\ref{lem:quatization}, we can establish the specific quantization error bounds for both the standard FL paradigm and our proposed FSL framework based on their respective precision assignments:
\begin{itemize}
    \item \emph{Standard FL} \cite{kim2023greenfl}: The entire model is trained locally using a uniform precision level, i.e., $q_i = q$ for all $i=1, \dots, d$. The error bound simplifies to:
    \begin{equation}
        \mathbb{E}[\norm{Q(\bs{w})-\bs{w}}^2] \le \frac{d}{2^{2q}}.
        \label{eq:fl_quant_err}
    \end{equation}
    
    \item \emph{Proposed FSL:} The model is partitioned into a client-side submodel $\bs{w}_\mathrm{c}\in \mathbb{R}^{\dc}$ and a server-side submodel $\bs{w}_\mathrm{s} \in \mathbb{R}^{\ds}$, quantized with precision levels $\qc$ and $\qs$, respectively. Let $\Qc(\cdot)$ and $\Qs(\cdot)$ denote the stochastic quantization functions, applied element-wise on $\bs{w}_\mathrm{c}$ and $\bs{w}_\mathrm{s}$, respectively. Therefore, the quantized weights can be written as $Q(\bs{w})=\sqb{\Qc\paren{\bs{w}_\mathrm{c}}^\itp, \Qs\paren{\bs{w}_\mathrm{s}}^\itp}^\itp$, and the error bound for the split architecture becomes:
    \begin{equation}
        \expec{\norm{Q(\bs{w})-\bs{w}}^2} \le \frac{\dc}{2^{2\qc}} + \frac{\ds}{2^{2\qs}}. \label{eq:fsl_quant_err}
    \end{equation}
\end{itemize}

Lemma~\ref{lem:quatization} demonstrates that our quantization scheme is unbiased. However, the quantization error variance remains bounded and dependent on the model dimension and the assigned precision level. Notably, \eqref{eq:fsl_quant_err} demonstrates that when the model is partitioned at early layers ($\dc \ll \ds$), the server strictly dominates the theoretical error bound. This yields a key structural advantage for GQ-FSL by decoupling device-side hardware limits from the overall quantization noise floor. Consequently, the framework can flexibly allocate precision between the client and server without severely compromising convergence. For example, raising $\qc$ at minimal energy cost to compensate for reduced server-side precision that drives substantial energy savings.

Next, we detail the specific local update process employed for the partitioned QNN.

\subsubsection{QNN Local Update Mechanism}
 Training the partitioned QNN requires maintaining full-precision  weights at both the clients and the edge server. For client $k$, we denote the full-precision weights of layer $l$ as $\bs{w}_k^{(l)}$, and the total number of layers is represented by $\ell$. The DNN is partitioned at the cut layer $\lc\in\{1,\dots,\ell\}$ such that the client-side submodel $\wkc$ comprises $\lc$ layers and the server-side submodel $\wks$ comprises $\ell-\lc$ layers. During the FP, the clients and server use quantized weights and activations corresponding to their respective precision levels ($\qc$ and $\qs$). Therefore, the quantized weights of the $l$-th layer are given by $\bs{\tilde{w}}_{k}^{(l)}=\Qc(\bs{w}_k^{(l)})$ for $l=1,\dots,\lc$, and $\bs{\tilde{w}}_{k}^{(l)}=\Qs(\bs{w}_k^{(l)})$ for $l=\lc+1,\dots,\ell$. Then, the output activations of layer $l$ are computed as: $ o^{(l)} = g^{(l)}(\bs{\tilde{w}}_{k}^{(l)}, o^{(l-1)})$, where $o^{(l-1)}$ is the output from the preceding layer, and $g^{(l)}(\cdot)$ encompasses the operations of layer $l$, including the linear transformation, normalization, and activation function. In particular, the activation function output is bounded to the QNN dynamic range and subjected to stochastic quantization before $o^{(l)}$ is fed into the subsequent layer as input. At the cut layer $\lc$, the quantized intermediate activations $o^{(\lc)}$ are transmitted over the wireless link to the edge server to complete the forward pass.

For the backward pass, BP is performed in full precision to effectively average out the stochastic gradient (SG) noise \cite{hubara2018quantized_nn}. Cut-layer gradients are returned to the client unquantized, and both entities update their respective full-precision weights $\bs{w}_k^{\tau}$ using mini-batch stochastic gradient descent (SGD), where $\tau$ is the current local iteration. 
Following the update step, the weights are clipped to the $[-1, 1]$ range as $\bs{w}_k^{\tau} \leftarrow \mathrm{clip}\paren{\bs{w}_k^{\tau}, -1, 1}$, where $\mathrm{clip}\paren{x, -1, 1}$ projects values larger than $1$ to $1$, values smaller than $-1$ to $-1$, and otherwise returns $x$. This ensures the updated full-precision weights remain within the valid dynamic range of the QNN. Finally, these clipped full-precision weights are re-quantized to produce $\bs{\tilde{w}}_{k}^{\tau}$ for the subsequent FP.

\subsection{GQ-FSL Learning Procedure}
The collaborative training process in our GQ-FSL framework proceeds in global communication rounds. Before training begins, the edge server splits the global model into client- and server-side submodels. At each global round $t$, the learning procedure follows these sequential steps:

\subsubsection{Client Selection and Model Distribution}
The BS selects a subset of $K$ clients, with heterogeneous computational capabilities, from the total $N$ available devices to participate in the current training round. We denote the set of participating clients at global round $t$ as $\mathcal{N}_t \subseteq \mathcal{N}$. 
The BS broadcasts the current global client-side submodel $\bs{w}_\mathrm{c}^{t}$ in full precision to the scheduled clients over the downlink channel.
Upon reception, each participating client $k \in \mathcal{N}_t$ initializes its local full-precision weights for the first local iteration as $\bs{w}_{k,\mathrm{c}}^{t,0} \leftarrow \bs{w}_\mathrm{c}^{t}$, while the edge server concurrently initializes a dedicated server-side submodel for each client as $\bs{w}_{k,\mathrm{s}}^{t,0} \leftarrow \bs{w}_\mathrm{s}^{t}$. 

\subsubsection{Local Split Training}
Following the distribution phase, each participating client $k \in \mathcal{N}_t$ collaborates with the edge server to perform local split training, where they execute $I$ local iterations of the forward and backward passes detailed in Section~\ref{sec:proposed_gqfsl}. Specifically, at each local iteration $\tau$, the client processes a mini-batch of data $\xi_k^{\tau}$ sampled from $\mathcal{D}_k$ and transmits the quantized intermediate activations, alongside the corresponding labels, to the edge server\footnote{While inherent to split-based frameworks, label sharing introduces potential privacy risks for sensitive data. Privacy-enhancing technologies like differential privacy can be integrated into our framework, though analyzing their impact remains an area for future work \cite{pham2024enhancingdp}.}. The server completes the FP, calculates the loss, and initiates BP by returning the cut-layer gradients to the client. Both entities then update their associated portion of the full-precision weights using the mini-batch SGD. Overall, at global round $t$ we have
\begin{equation}
    \bs{w}_k^{t,\tau} = \bs{w}_k^{t,\tau-1} - \eta_t \nabla F_k\paren{\bs{\tilde{w}}_{k}^{t,\tau-1},\xi_k^{\tau}}, \ \tau=1,...,I, 
\end{equation}
with
\begin{equation}
     \bs{w}_k^{t,\tau}=\begin{bmatrix}
                \bs{w}_{k,\mathrm{c}}^{t,\tau} \vspace{0.2mm}\\
                \bs{w}_{k,\mathrm{s}}^{t,\tau}
             \end{bmatrix}, \quad
             \bs{\tilde{w}}_{k}^{t,\tau}=\begin{bmatrix}
                \bs{\tilde{w}}_{k,\mathrm{c}}^{t,\tau} \vspace{0.2mm} \\
                \bs{\tilde{w}}_{k,\mathrm{s}}^{t,\tau}
             \end{bmatrix} = \begin{bmatrix}
                \Qc(\bs{w}_{k,\mathrm{c}}^{t,\tau}) \vspace{0.2mm} \\
                \Qs(\bs{w}_{k,\mathrm{s}}^{t,\tau})
             \end{bmatrix},     
\end{equation}
where $\eta_t$ is the corresponding learning rate.

\subsubsection{Model Global Aggregation} 
After completing $I$ local iterations, the updated local models must be aggregated to form the new global model for round $t+1$. We adopt FedAvg \cite{mcmahan2017communication} to aggregate the multiple local models associated with each client. 
Prior to uplink transmission, each client $k \in \mathcal{N}_t$ calculates the submodel update (also known as weight update/differential) $\bs{d}_{k,\mathrm{c}}^{\,t+1}= \bs{w}_{k,\mathrm{c}}^{t+1}-\bs{w}_{k,\mathrm{c}}^{t}$, where $\bs{w}_{k,\mathrm{c}}^{t+1}=\bs{w}_{k,\mathrm{c}}^{t,I}$ and $\bs{w}_{k,\mathrm{c}}^{t}=\bs{w}_{k,\mathrm{c}}^{t,0}$ \cite{zheng2020design}. Since $\bs{d}_{k,\mathrm{c}}^{\,t+1}$ typically has millions of elements even for a partial DNN, it is not practical to transmit it with full precision for energy-constrained devices. As in \cite{kim2023greenfl}, we therefore apply the quantization scheme introduced in Section~\ref{sec:proposed_qnn}, using a transmission precision level $\qu$ to get its quantized equivalent $\bs{\tilde{d}}_{k,\mathrm{c}}^{\raisebox{-0.6ex}{$\scriptstyle \, t+1$}}=Q_\mathrm{u}(\bs{d}_{k,\mathrm{c}}^{\,t+1})$ which is transmitted to the BS. Note that prior to quantization, the $k$-th client normalizes the model update of round $t$ as $\bs{d}_{k,\mathrm{c}}^{\,t+1}/ \|\bs{d}_{k,\mathrm{c}}^{\,t+1}\|_{\infty}$ to match the predetermined quantization range $\sqb{-1,1}$.
Upon reception of all quantized client submodel updates, the parameter server re-normalizes them to their true magnitude, and aggregates the results using the FedAvg algorithm to generate the updated global client-side submodel.
Concurrently, since the server-side submodels never leave the edge server, they are aggregated directly in full precision without undergoing transmission quantization using $\bs{d}_{k,\mathrm{s}}^{\,t+1}= \bs{w}_{k,\mathrm{s}}^{t+1}-\bs{w}_{k,\mathrm{s}}^{t}$.
Overall, the parameter server generates the next global model $\bs{w}^{t+1}$ as
\begin{equation}
    \bs{w}^{t+1} = \bs{w}^t + \frac{1}{K}\sum_{k \in \mathcal{N}_t} \bs{\tilde{d}}_{k}^{\raisebox{-0.6ex}{$\scriptstyle \, t+1$}}, \label{eq:model_agg}
\end{equation}
where 
\begin{equation}
    \bs{w}^t = 
    \begin{bmatrix}
        \bs{w}_\mathrm{c}^t \\
        \bs{w}_\mathrm{s}^t
    \end{bmatrix}, 
    \quad
    \bs{\tilde{d}}_{k}^{\raisebox{-0.6ex}{$\scriptstyle \, t+1$}} = \begin{bmatrix}
        \bs{\tilde{d}}_{k,\mathrm{c}}^{\raisebox{-0.6ex}{$\scriptstyle \, t+1$}} \vspace{0.2mm} \\
        \bs{d}_{k,\mathrm{s}}^{\,t+1}
    \end{bmatrix},
\end{equation}
and clips it to the $[-1,1]$ range to ensure valid full-precision weights for the subsequent round.

This iterative process of model distribution, split training, and aggregation repeats until the global loss function converges to a predefined target accuracy $\epsilon$. With the operational mechanics of the GQ-FSL framework formally established, the subsequent analysis models the physical resource footprint required to execute these computational and transmission phases across the resource-constrained wireless edge.

\begin{figure}[t]
     \centering
     \includegraphics[width=\linewidth]{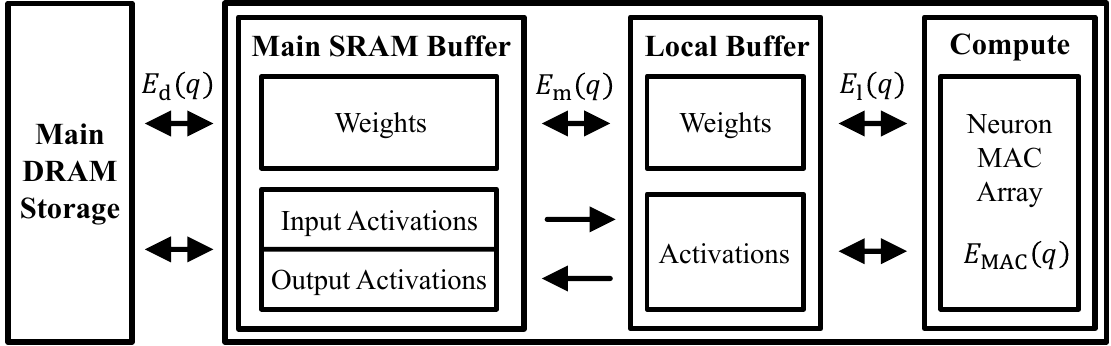}
     \caption{An illustration of the two-dimensional processing chip architecture, modified from \cite{moons2019embedded_qnn}.}
     \label{fig:chip}
\end{figure}

\section{GQ-FSL Computation \& Transmission Models} \label{sec:energy_latency}
To quantify the resource footprint of the proposed framework, we develop split-based parameterized models for the EC and latency associated with the GQ-FSL collaborative training procedure. To formulate these models, we first define the physical architecture of the computational platform.

Following the hardware architecture used in \cite{kim2023greenfl}, we consider a generic two-dimensional processing chip (e.g., a DNN accelerator) as shown in Fig.~\ref{fig:chip} \cite{moons2017minimum_qnn, moons2019embedded_qnn}. Deployed at the $k$-th client, the chip operates at a clock frequency $f_k$, and comprises a large off-chip dynamic random-access memory (DRAM), a parallel neuron array with an area for $\pkc$ MAC units, and a two-tiered on-chip memory hierarchy. This on-chip memory consists of a small local buffer (register file) that caches currently used operands, and a main static random-access memory (SRAM) buffer of size $\Skc$. The main SRAM is evenly partitioned, allocating $\Skc/2$ bits for storing the client-side submodel weights and $\Skc/2$ bits for the activations \cite{moons2017minimum_qnn}. Because the main SRAM buffer has limited capacity, the local dataset $\mathcal{D}_k$ is stored in the DRAM. Additionally, if the entire client-side submodel does not fit within the allocated main SRAM, the overflow weights are stored in the DRAM. 
We complement the client model by considering a high-performance, multi-core computing platform at the edge server. We assume that the edge server employs resource virtualization, allocating to each participating client a dedicated computational instance equipped with $\ps$ MAC units, a main SRAM capacity of $S$ bits, and an operating clock frequency $f_\mathrm{s}$. This parallel allocation allows the server to process the server-side submodels for all active clients simultaneously without queueing delays.


In our GQ-FSL framework, the computational load is physically distributed between the client and the edge server, strictly dictated by the cut layer $\lc$. Consequently, the network parameters are split accordingly. Let $N_l$, $A_l$, and $W_l$ denote the number of MAC operations, output activations, and weights at the $l$-th layer, respectively. The client-side computational load for a single data sample is parameterized by $\Mc(\lc) = \sum_{l=1}^{\lc}N_l$, $\Oc(\lc) = \sum_{l=1}^{\lc}A_l$, and $\dc(\lc) = \sum_{l=1}^{\lc}W_l$
. The size of the intermediate activations transmitted at the cut layer is given by $x_\mathrm{cut}(\lc) = A_{\lc} $. The server handles the remaining architecture, yielding $\Ms(\lc) = M - \Mc(\lc)$, $\Os(\lc) = O - \Oc(\lc)$, and $\ds(\lc) = d - \dc(\lc)$, where $M$, $O$, and $d$ are the corresponding parameters of the full DNN. For notational simplicity in the remainder of this paper, we hereafter omit the explicit dependence on $\lc$, denoting these parameters simply as $\Mc, \Oc, \dc, x_\mathrm{cut}$ and their server-side counterparts as $\Ms, \Os, \ds$.

\subsection{Latency Model} \label{sec:latency}
In alignment with existing SL-based frameworks \cite{lin2024efficient_psl, wu2023split}, we consider the per-global-round latency, which encompasses the computation\footnote{We ignore the model aggregation latency in the computation analysis because the workload of FedAvg aggregation is negligible compared to the intensive matrix multiplications involved in model training \cite{wu2023split}.} 
and communication 
delay
.

To model the computation latencies, we consider the computing chips clock frequencies $f_k$ and $f_\mathrm{s}$ for the $k$-th client and edge server, respectively. The computational throughput is governed by the effective parallelism of the two-dimensional MAC array, which dictates the actual number of MAC operations executed per clock cycle. This effective parallelism factor, denoted by $\pkc^\mathrm{eff}(\qc) = \pkc\cdot(\qmax/\qc)$ for the client and $\ps^\mathrm{eff}(\qs) = \ps\cdot(\qmax/\qs)$ for the server, increases at lower precisions because a physical silicon area containing $\pkc$ full-precision ($q_{\max}$-bit) MACs can accommodate proportionally more reduced-precision MACs \cite{moons2017minimum_qnn}. Therefore, the client- and server-side FP latencies per data sample can be calculated, respectively, as 
\begin{equation}
    \tfpkc(\qc,\lc) = \frac{\Mc}{\pkc^\mathrm{eff}(\qc)f_k}, \quad \tfps(\qs,\lc) = \frac{\Ms}{\ps^\mathrm{eff}(\qs)f_\mathrm{s}}. \label{eq:latency_fp}
\end{equation}

During BP, gradients are calculated in full precision ($\qmax$)
. For each layer, computing the gradients with respect to both its weights and the preceding layer's activations requires approximately twice the number of MAC operations as the FP (i.e., $2\Mc$ for the client) \cite{kim2023greenfl}.
Consequently, the client- and server-side BP latencies per data sample are given by
\begin{equation}
    \tbpkc(\lc) = \frac{2\Mc}{\pkc f_k}, \quad \tbps(\lc) = \frac{2\Ms}{\ps f_\mathrm{s}}, \label{eq:latency_bp}
\end{equation}
where using $\qmax$ inherently reduces the effective parallelism factor to the baseline physical array size $\pkc$ and $\ps$.

Having established the latency models for the local computations, we now account for the wireless communication latencies. In our FSL framework, communication occurs in two distinct phases: the per-local-iteration exchange of cut-layer data during the FP and BP, and the per-global-round submodel uploading and broadcasting.
For the wireless link between the client and the BS, we assume an orthogonal frequency division multiple access (OFDMA) scheme, where each participating device is allocated dedicated resources. As in \cite{kim2023greenfl}, the achievable uplink transmission rate of client $k$ is
\begin{equation}
    \rku = B \log_2\paren{1 + \frac{P_k^\mathrm{tx} h_k}{N_0 B}}, \label{eq:rate_up}
\end{equation}
where $B$ is the allocated bandwidth, $P_k^\mathrm{tx}$ is the transmission power, $N_0$ is the white noise power spectral density, and $h_k\!$ represents the channel gain, assumed constant during training. Under channel reciprocity, the achievable downlink transmission rate $\rkd$ is formulated analogously by replacing $P_k^\mathrm{tx}$ with the server's per-client allocated transmission power $P_\mathrm{s}^\mathrm{tx}$.

In a local iteration, the $k$-th client transmits the quantized intermediate activations at the cut layer to the server during the FP, and the server returns the corresponding full-precision gradients during the subsequent BP.
For a single data sample, the resulting uplink and downlink transmission latencies are thus given, respectively, by
\begin{equation}
    \tka(\qc,\lc) = \frac{x_{\mathrm{cut}} \qc}{\rku}, \quad \tkg(\lc) = \frac{x_{\mathrm{cut}} \qmax}{\rkd}, \label{eq:latency_act_grad}
\end{equation}
where $x_{\mathrm{cut}} \qc$ bits and $x_{\mathrm{cut}} \qmax$ bits are the respective uplink and downlink payload sizes. 


At the conclusion of a global round $t$, the $k$-th client uploads its $\qu$-bit quantized local submodel update $\bs{\tilde{d}}_{k}^{\raisebox{-0.6ex}{$\scriptstyle \, t+1$}}$ to the server for aggregation, requiring an upload latency of
\begin{equation}
    \tku(\qu,\lc) = \frac{\dc \qu}{\rku}. \label{eq:latency_up}
\end{equation}
Following global aggregation, the server broadcasts the updated global client-side submodel in full precision back to the participating clients. Using broadcasting bandwidth $B_b$ and transmit power $P_b^\mathrm{tx}$, the downlink broadcasting data rate is 
\begin{equation}
    \rb = B_b \log_2\paren{1 + \frac{P_b^\mathrm{tx} h_{\min}}{N_0 B_b}}, \label{eq:rate_brod}
\end{equation}
where $h_{\min}=\min_{k\in\mathcal{N}_t}(\hspace{0.1mm}h_k)$ represents the weakest channel gain among the participants. Accordingly, the corresponding broadcasting latency is 
\begin{equation}
    \tb(\lc) = \frac{\dc \qmax}{\rb}. \label{eq:latency_brod}
\end{equation}

\subsection{Energy Consumption Model}
The total EC of our proposed FSL framework encompasses both the local computations executed on the hardware accelerators and the wireless communication required to bridge the split architecture and exchange submodels. We first establish a general computation EC model of the FP and BP for a single data sample across both the clients and the server, followed by the communication EC model.

To quantify the computational cost across both entities, we adopt a precision-dependent energy model for the MAC operations \cite{moons2019embedded_qnn, kim2023greenfl}. For a given precision level $q$, the energy consumed by a single MAC operation is modeled as $\Emac(q) = A (q/\qmax)^a$, where $A > 0$ is a hardware-specific scaling constant, $1 < a < 2$ is a non-linear scaling factor, and $\qmax$ is the full-precision level. Because data movement between the various memory layers and the off-chip DRAM has a dominant impact on the energy footprint \cite{horowitz201410_EC, yang2017method_dataflow}, these memory accesses must be accounted for in the computation EC model. As in \cite{moons2017minimum_qnn}, the EC of reading from or writing to the local buffer is modeled as $\El(q)=\Emac(q)$, and for the main buffer as $\Em(q)=2\Emac(q)$. Since fetching data from the DRAM incurs a significantly higher cost than a single equivalent MAC operation \cite{horowitz201410_EC}, it is modeled as $\Ed(q) = A_{\mathrm{d}} \Emac(q)$, with $A_{\mathrm{d}} \gg 1$. We further assume that bias additions, normalization, and activation functions, which are performed in full precision \cite{hubara2018quantized_nn}, are modeled to cost $\Emac(\qmax)$ \cite{moons2019embedded_qnn}.

For the $k$-th client operating at a quantization precision level $\qc$, the EC for the FP computation (i.e., inference) per data sample is the sum of the energy consumed by the MAC units $\Eckc(\qc,\lc)$, weight memory accesses $\Ewkc(\qc,\lc)$, activation memory accesses $\Eakc(\qc,\lc)$, and DRAM accesses $\Edramkc(\qc,\lc)$, which yields \cite{moons2017minimum_qnn,moons2019embedded_qnn}
\begin{equation}
    \Efpkc(\qc,\lc) = \Eckc+\Ewkc+\Eakc+\Edramkc,
\end{equation}
where each component is modeled as a function of both the precision level and the splitting point:
\begin{align}
    \Eckc(\qc,\lc) &= \Emac(\qc)\Mc +3\Oc\Emac(\qmax), \nonumber \\
    \Ewkc(\qc,\lc) &= \Em(\qc)\dc + \El(\qc)\frac{\Mc}{\sqrt{\pkc^\mathrm{eff}(\qc)}}, \nonumber \\
    \Eakc(\qc,\lc) &= 2\Em(\qc)\Oc + \El(\qc)\frac{\Mc}{\sqrt{\pkc^\mathrm{eff}(\qc)}}, \nonumber \\
    \Edramkc(\qc,\lc) \! &= \Ed(\qmax)x_\mathrm{in} \!+ \Ed(\qc)\!\max\paren{\hspace{-0.2pt}\!\dc\qc\!-\!\frac{\Skc}{2},0\!\hspace{-0.2pt}} \nonumber \\
    & \quad\, + 2\Ed(\qc)\max\paren{\!\Oc\qc-\frac{\Skc}{2},0\!},
\end{align}
and $x_\mathrm{in}$ is the input dimension. The factor of 3 in $\Eckc$ accounts for the equivalent MAC costs of additive biasing, normalization, and activation functions, which are performed on all $\Oc$ output activations. For the on-chip memory accesses ($\Ewkc$ and $\Eakc$), data is transferred from the main SRAM to the local buffer for extensive reuse. Note that accessing the main buffer incurs a factor of 2 for activations, as they must be both fetched and subsequently stored back for the calculations of the next layer, whereas weights are only fetched. Due to the two-dimensional structure of the MAC array, activation- and weight-level parallelism enable a single weight to be reused simultaneously across $ \sqrt{\smash{\textstyle p_{\raisebox{-0.4ex}{$ \scriptstyle k,\mathrm{c}$}}}^{\hspace{-3.7mm}\mathrm{eff}}}\,$ activations, and vice versa \cite{moons2017minimum_qnn, moons2019embedded_qnn}. 
Finally, in $\Edramkc$, the raw input features $x_\mathrm{in}$ are fetched in full precision. For data exceeding the allocated main SRAM capacity, we similarly apply a factor of 2 solely to the intermediate output activations, which must be both read and written to the off-chip DRAM. Overflow weights are read-only as a full-precision copy is maintained in the DRAM.

The server-side EC for the FP, denoted as $\Efpks(\qc,\qs, \lc)$, follows an analogous structure, substituting the client-specific parameters with $\qs$, $\ps$, $\ps^\mathrm{eff}(\qs)$, $S$, $\Ms$, $\Os$, and $\ds$. A distinction, however, lies in the DRAM access energy. Instead of fetching the raw input data $x_{\mathrm{in}}$, the server fetches the cut-layer quantized activations $x_{\mathrm{cut}}$. Thus, the server's DRAM energy is formulated as
\begin{align}
    \Edramks(\qc,\qs, \lc) \! &= \Ed(\qc) x_{\mathrm{cut}}  + \Ed(\qs)  \max\paren{\!\ds \qs - \!\frac{S}{2}, 0\!} \nonumber \\
    & \quad\,+ 2\Ed(\qs) \max\paren{\!\Os \qs - \frac{S}{2}, 0\!}.
\end{align}

Since BP is executed in full precision and requires approximately twice as many MAC operations as FP, the EC for the BP at the client side is given by \cite{kim2023greenfl}
\begin{align}
    \Ebpkc(\lc) &= 2\Mc\Emac(\qmax) + \paren{\dc + 2\Oc}\Em(\qmax) \nonumber\\ 
      +&\, 2\El(\qmax)\frac{2\Mc}{\sqrt{\pkc}} \!+ \!\Ed(\qmax)  \!\max\paren{\!\dc \qmax \! - \!\frac{\Skc}{2}, 0\!}\nonumber \\
      +&\, 2\Ed(\qmax) \max\paren{\!\Oc \qmax - \frac{\Skc}{2}, 0\!},
\end{align}
The server-side BP energy, $\Ebpks(\lc)$, is formulated analogously with the corresponding server parameters. Therefore, the overall local computation EC for a single data sample is 
\begin{equation}
    E_k^\mathrm{C}(\qc,\qs,\lc)= E^\mathrm{FP}_k(\qc,\qs,\lc) + E^\mathrm{BP}_k(\lc), \label{eq:energy_comp}
\end{equation}
where $E^\mathrm{FP}_k(\qc,\qs,\lc) = \Efpkc(\qc,\lc) + \Efpks(\qc,\qs,\lc)$, and $E^\mathrm{BP}_k(\lc) = \Ebpkc(\lc) + \Ebpks(\lc)$.

Having established the EC model for the local computations, we now account for the wireless communication overhead incurred during the per-local-iteration exchange of intermediate data and the per-global-round submodel uploading and broadcasting. Since the transmission delays were already defined in Section~\ref{sec:latency}, the corresponding transmission EC is directly formulated by scaling the respective latencies by the designated transmission powers. The energy consumed by the intermediate cut-layer communication to complete the FP and BP for a single data sample is
\begin{equation}
    E_k^\mathrm{T}(\qc,\lc) = \tka(\qc,\lc)P_k^\mathrm{tx} + \tkg(\lc)P_\mathrm{s}^\mathrm{tx}, \label{eq:energy_comm}
\end{equation}
where the transmission latencies of the intermediate activations $\tka(\qc,\lc)$ and gradients $\tkg(\lc)$ are given in Eq.~\eqref{eq:latency_act_grad}.
Note that the EC of a single SGD step, comprising both the FP and BP, is the sum of the per-data sample energies $E_k^\mathrm{C}$ and $E_k^\mathrm{T}$. During training, this computation is performed at every local iteration over a mini-batch of data samples. 

Finally, the EC for uploading the quantized local submodel for aggregation at the conclusion of a global round is 
\begin{equation}
    \Eku(\qu,\lc) = \tku(\qu,\lc) P_k^\mathrm{tx}, \label{eq:energy_up}
\end{equation}
where the submodel upload latency for the $k$-th client $\tku(\qu,\lc)$ is defined in Eq.~\eqref{eq:latency_up}, and the EC incurred by the server during the broadcasting phase is
\begin{equation}
    \Eb(\lc) = \tb(\lc) P_b^\mathrm{tx}, \label{eq:energy_brod}
\end{equation}
with broadcasting latency $\tb(\lc)$ as in Eq.~\eqref{eq:latency_brod}.

\section{Energy Efficient QNN-based FSL} \label{sec:problem_formulation}
Deploying GQ-FSL effectively requires balancing the energy savings of reduced precision and computation offloading against their impact on training latency and convergence. To navigate this tradeoff, we develop a joint optimization strategy that minimizes the total energy footprint under strict training latency and target accuracy constraints. We first establish the synchronous per-round latency, followed by formulating the core optimization problem. Next, we derive a theoretical convergence bound to characterize the target accuracy under statistically heterogeneous data. Finally, we transform the problem into tractable analytical expressions to determine the optimal framework configuration.

\subsection{Per-Round Total Latency}
Prior quantized FL research typically navigates the framework design by analyzing the tradeoff between EC and the total number of global rounds $T$ required for convergence. However, optimizing purely for global rounds ignores the physical time required to execute them. In a practical split architecture, lowering bit-widths and adjusting the cut layer profoundly alters both the per-round computation and transmission delays. 




\begin{figure}[t]
     \centering
     \includegraphics[width=\linewidth]{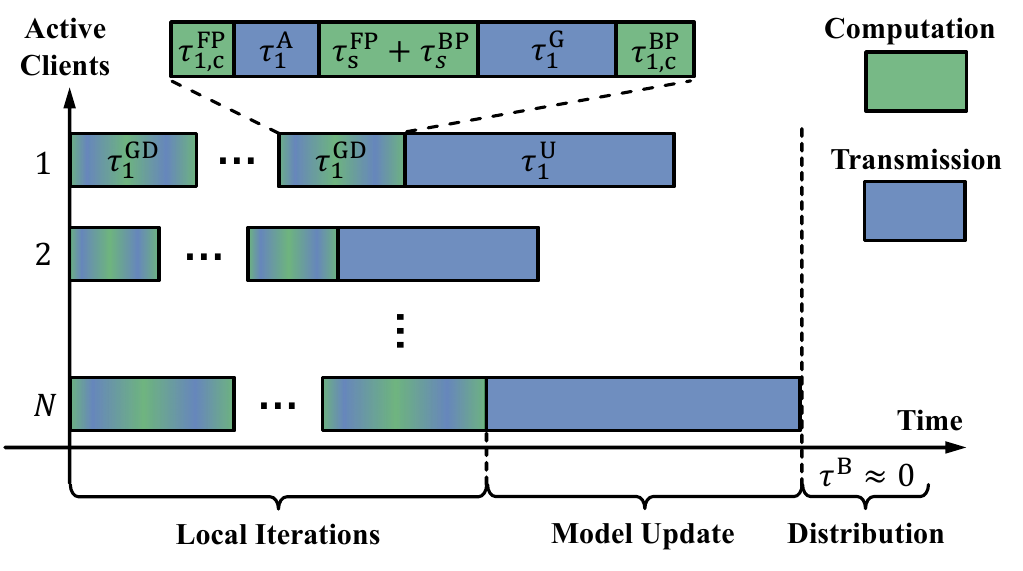}
     \caption{FSL per-round latency breakdown ($b=1$).}
     \label{fig:latency}
\end{figure}

To capture the total training time, we must formalize the physical execution of each round. Because global aggregation is synchronous, the per-round latency is strictly bottlenecked by the slowest participating client, a limitation commonly known as the straggler effect. As shown in Fig.~\ref{fig:latency}, the total latency per global round\footnote{Note that the latency associated with model broadcasting in Eq.~\eqref{eq:latency_brod} is neglected from Eq.~\eqref{eq:latency_round} due to the high transmit power and large bandwidth typically available at the BS \cite{yang2020energy_efficient_fl}.} can be expressed as
\begin{equation}
    \tau(\qc,\qs,\qu,\lc) = \max_{k\in\mathcal{N}_t}\curly{\tku(\qu,\lc) + \!Ib\,\tau^\mathrm{GD}_k(\qc,\qs,\lc)}, \label{eq:latency_round}
\end{equation} 
where $I$ is the number of local iterations, $b$ is the mini-batch size, and 
\begin{align}
    \tau^\mathrm{GD}_k(\qc,\qs,\lc) &= \tfpkc(\qc,\lc)+\tka(\qc,\lc)+ \tfps(\qs,\lc) \nonumber \\
    & \hspace{4mm}+\tbps(\lc)+\tkg(\lc)+\tbpkc(\lc).
\end{align}

\subsection{Problem Formulation}
Given our established energy and latency models, we now formulate a joint optimization problem to minimize the expected total training EC of the proposed GQ-FSL framework, while ensuring convergence under a target accuracy $\epsilon$ and an expected training latency constraint $\tau_{\max}$. The optimization problem is formulated as follows:
\begin{subequations}\label{eq:opt_problem}
    \begin{align}
        \min_{\qc, \qs, \qu, \lc, K, I} \, & \mathbb{E} \Bigg[ \sum_{t=1}^{T} \bigg[ \Eb(\lc) + \sum_{k \in \mathcal{N}_t} \Big( \Eku(\qu,\lc) \nonumber \\ 
        & \hspace{6mm} + I b \big( E_k^\mathrm{C}(\qc,\qs,\lc) + E_k^\mathrm{T}(\qc,\lc) \big) \Big) \bigg] \Bigg] \label{eq:opt_obj} \\
        \text{s.t.} \quad 
        &  \expec{F(\bs{w}^{\text{\raisebox{-0.4ex}{$\scriptstyle T$}}})} - F(\bs{w}^*) \le \epsilon, \label{eq:opt_acc} \\
        & T\,\expec{\tau(\qc,\qs,\qu,\lc)} \le \tau_{\max}, \label{eq:opt_lat} \\
        & \lc \in \{1, \dots, \ell\}, \label{eq:opt_split} \\
        & \qc, \qs, \qu \in \{1, \dots, \qmax\}, \label{eq:opt_quant} \\
        & K \in \{1,\dots,N\}, \, I \in \{1,\dots,I_{\max}\} \label{eq:opt_client_iter},
    \end{align}
\end{subequations}
where $F(\bs{w}^{\text{\raisebox{-0.4ex}{$\scriptstyle T$}}})$ is the global loss function evaluated at the final model weights $\bs{w}^{\text{\raisebox{-0.4ex}{$\scriptstyle T$}}}$ after $T$ global rounds, $F(\bs{w}^*)$ is the optimal loss, and $\tau(\qc,\qs,\qu,\lc)$ is the per-round latency defined in Eq.~\eqref{eq:latency_round}. Constraint~\eqref{eq:opt_acc} ensures that the trained FSL model converges to the target accuracy $\epsilon$ after $T$ rounds. Constraint~\eqref{eq:opt_lat} enforces that the expected training latency does not exceed $\tau_{\max}$. Constraint~\eqref{eq:opt_split} restricts the cut layer to a valid intermediate layer index, ensuring a partition into client- and server-side submodels that prevents complete model offloading. Constraint~\eqref{eq:opt_quant} restricts the decision variables for the client computations, server computations, and submodel uplink transmissions to valid integer quantization levels up to the maximum precision $\qmax$. Finally, constraint~\eqref{eq:opt_client_iter} bounds the number of participating clients up to the total available devices $N$ and caps the local iterations at a predefined maximum $I_{\max}$.

Problem~\eqref{eq:opt_problem} is formulated 
under the assumption that the channel gains $h_k$ and the heterogeneous computational capabilities of the devices ($f_k,\pkc,\Skc$) are 
known. Then, 
the expectations in the EC objective function~\eqref{eq:opt_obj} and the latency constraint~\eqref{eq:opt_lat} 
are taken 
with respect to the random sampling of $K$ clients per global round. For the target accuracy constraint~\eqref{eq:opt_acc}, the expectation is taken with respect to the random client selection, the stochastic quantization noise, and the SG noise. 

\subsection{Convergence Analysis} \label{sec:convergence}
To solve problem~\eqref{eq:opt_problem}, we must first tractably express the target accuracy constraint in~\eqref{eq:opt_acc} by analyzing the convergence behavior of our proposed GQ-FSL framework. This analysis establishes an analytical relationship between the expected optimality gap and the decision variables,
allowing us to explicitly derive the required number of global rounds $T$ with respect to the target accuracy $\epsilon$. 

Following \cite{zheng2020design}, we simplify the analysis by assuming uniform local dataset sizes across all clients, $D_k \!=\! D/N$ for all $k\in\mathcal{N}$, and by considering the scenario where the $K$ participating clients are sampled uniformly at random during each global round. 
To quantify the degree of data heterogeneity (non-IID), we define $\Gamma=F(\bs{w}^*)-\frac{1}{N}\!\sum_{k=1}^NF_k^*$, where $F(\bs{w}^*)$ and $F_k^*$ denote the global and local optimums, respectively \cite{li2020flconvergence,zheng2020design}.
Consistent with standard practices in the literature \cite{li2020flconvergence, zheng2020design, kim2023greenfl}, we introduce the following assumptions regarding the loss function and the SG.
\begin{assumption} \label{asm:loss_sg}
    The loss functions satisfy the following properties:
    \begin{itemize}
        \item $F_k(\bs w)$ is L-smooth: $\forall \bs v$ and $\bs w$, \vspace{-2.5mm}
        \begin{equation*}
            F_k(\bs v) \leq F_k(\bs w)+(\bs v -\bs w)^\tp \nabla F_k(\bs w) +\frac{L}{2}\norm{\bs v - \bs w}^2. \vspace{-2mm}
        \end{equation*} 
        
        \item $F_k(\bs w)$ is $\mu$-strongly convex: $\forall \bs v$ and $\bs w$, \vspace{-2.5mm}
        \begin{equation*}
            F_k(\bs v) \geq F_k(\bs w)+(\bs v -\bs w)^\tp \nabla F_k(\bs w) +\frac{\mu}{2}\norm{\bs v - \bs w}^2. \vspace{-2mm}
        \end{equation*}
        
        \item Bounded variance for mini-batch SG: \vspace{-2.5mm}
        \begin{equation*}
            \mathbb{E}\left[ \norm{\nabla F_k(\wkt, \xi_{k}^t) -  \nabla F_k(\wkt)}^2 \right] \! \leq \sigma_k^2, \hspace{-3.2pt} \quad k=1,\dots,N. \vspace{-2mm}
        \end{equation*}
    
        \item Uniformly bounded squared norm of mini-batch SG: \vspace{-2.5mm}
        \begin{equation*}
            \hspace{-16mm} \mathbb{E}\left[ \norm{\nabla F_k(\wkt,\xi_k^t)}^2\right] \leq G^2, \quad  k=1,\dots,N.
        \end{equation*}
    \end{itemize}
\end{assumption}

Stochastic quantization during local training and wireless transmission inherently induces noise that degrades model accuracy and convergence of our FSL Framework. 
The following theorem leverages bound~\eqref{eq:fsl_quant_err} from Lemma~\ref{lem:quatization} to establish a generalized upper bound on the expected optimality gap, explicitly capturing the interplay between the quantization errors and the FSL architecture.

\begin{theorem}[] \label{thm:convergence}
    Let Assumption~\ref{asm:loss_sg} hold. Consider a learning rate $\eta_t=\frac{2}{\mu(t+\gamma)}$ where $\gamma=\max\big(8\frac{L}{\mu}-1,I\big)$, and let $\alpha>L$ satisfying $\eta_1\le\frac{\alpha-L}{\alpha L}$.
    Then, the optimality gap of the proposed GQ-FSL framework satisfies
    \begin{equation}
      \expec{F(\bs{w}^{\text{\raisebox{-0.4ex}{$\scriptstyle T$}}})}\! - \!F(\bs{w}^*) \leq \frac{L}{2(TI +\gamma)}\left[\frac{4\psi_2 + (\gamma+1)G^2}{\mu^2}\right] \!\! + \! \frac{L\psi_1}{2\mu \rule[-1.5ex]{0pt}{0pt}}, \label{eq:converge}
    \end{equation}
    where \vspace{-1.5mm}
    \begin{IEEEeqnarray}{rCl}
        \psi_1 &=& (\alpha - \mu)\left( \frac{\dc}{2^{2\qc}} + \frac{\ds}{2^{2\qs}} \right) + \frac{2L\Gamma}{\alpha}, \label{eq:psi1}
        \\
        \psi_2 &=& \frac{8\alpha(I-1)^2G^2}{\mu(\gamma+1)} + 2L\Gamma + \sum_{k=1}^N \frac{\sigma_k^2}{N^2} \nonumber \\
        && + \frac{4I^2G^2(N-K)}{K(N-1)} + \frac{4\dc IG^2}{K2^{2\qu}}. \label{eq:psi2}
    \end{IEEEeqnarray}
\end{theorem}
\begin{IEEEproof}
     The proof is provided in Appendix~\ref{app:thm_1_proof}. 
\end{IEEEproof}

\textit{Remark 1:} Our convergence analysis builds upon the framework established in \cite{kim2023greenfl} for QNN-based FL. An in-depth examination of the proof in \cite{kim2023greenfl} reveals  the need to sharpen the argument, particularly the explicit assumption of a lower bound on the learning rate ($\rho\eta_t\ge1$ for some $\rho\gg1$), which contradicts the learning rate specified in their main theorem ($\eta_t \le 1/\rho$) and yields an imprecise convergence bound. To resolve this, 
we introduce a tunable parameter $\alpha$ into the convergence analysis in Appendix~\ref{app:local_iter}. This bypasses the restrictive assumption and yields a strictly tighter upper bound on the optimality gap induced by quantization errors, as detailed later in this section. We further refine the induction proof in Appendix~\ref{app:induction} to better accommodate the optimality gap resulting from quantization errors during training, and to include the missing factor that addresses the induction base case. 
Moreover, 
by explicitly evaluating the impact of the splitting point, our generalized bound encompasses both standard quantized FL and the new quantized FSL architectures. 

\textit{Remark 2 (Convergence Properties):} 
Theorem~\ref{thm:convergence} establishes the fundamental relationship between our GQ-FSL architectural parameters and the resulting convergence bound. We observe several intuitive monotonic relationships within this theoretical guarantee. Increasing the number of participating clients $K$ tightens the bound, thus reducing the optimality gap. 
In contrast, the quantization error components embedded within $\psi_1$ and $\psi_2$ decay exponentially as the precisions $\qc$, $\qs$ and $\qu$ increase. 
The convergence bound is also degraded by the data heterogeneity factor $\Gamma$, whose impact, together with that of quantization, is examined in greater detail in the subsequent analysis.
Finally, while theoretical convergence can be strictly improved by maximizing precision levels or client participation, these improvements come at the cost of inflating the system's overall EC and training latency.

\subsubsection{Optimal \texorpdfstring{$\alpha$}{alpha} Selection}
The convergence upper bound of the proposed GQ-FSL framework depends on the free parameter $\alpha$ introduced in Theorem~\ref{thm:convergence}. 


\begin{corollary} \label{cor:optimal_alpha}
    The parameter $\alpha^*$ that minimizes the convergence upper bound in Theorem~\ref{thm:convergence} is given by
    \begin{equation} \label{eq:alpha_star}
    \alpha^* = \max\left\{\sqrt{\frac{2L\Gamma(TI+\gamma)}{X Y \mu + Z(TI+\gamma)}}, \, \frac{\mu L(\gamma+1)}{\mu(\gamma+1)-2L }, \, L \right\} \rule[-5.5mm]{0pt}{0pt}, 
    \end{equation}
    where $X=\frac{4}{\mu^2}$, $Y=\frac{8(I-1)^2G^2}{\mu(\gamma+1)}$, and $Z=\frac{\dc}{2^{2\qc}}+\frac{\ds}{2^{2\qs}}$. \vspace{1mm}
\end{corollary}
\begin{IEEEproof}
    Please refer to Appendix~\ref{app:cor_1_proof}.
\end{IEEEproof}
Next, we analyze the asymptotic behavior of the optimality gap as the number of global rounds $T$ and the compute quantization precisions $(\qc, \qs)$ increase.

\subsubsection{Asymptotic Optimality Gap}
Based on Theorem~\ref{thm:convergence}, the optimality gap converges to the following limit: 
\begin{equation}
    \expec{F(\bs{w}^{\text{\raisebox{-0.4ex}{$\scriptstyle T$}}})}\! - \!F(\bs{w}^*) \to \frac{L\psi_1}{2\mu} \quad \text{as} \quad T\to\infty, \label{eq:optimal_gap}
\end{equation}
where $\psi_1 = (\alpha^* - \mu)Z + \frac{2L\Gamma}{\alpha^*}$ consists of two components. The first term is driven directly by the local training quantization error, $Z=\left( \frac{\dc}{2^{2\qc}} + \frac{\ds}{2^{2\qs}} \right)$, as established in Eq.~\eqref{eq:fsl_quant_err}, while the second term is driven by the degree of data heterogeneity. 
Guaranteeing a specific target accuracy $\epsilon$ as per~\eqref{eq:opt_acc} requires the asymptotic bound in~\eqref{eq:optimal_gap} to satisfy $\frac{L\psi_1}{2\mu} \le \epsilon$. Notably, this condition inherently establishes minimum threshold levels for the local precisions $\qc$ and $\qs$.

As shown in Eq.~\eqref{eq:alpha_star} of Corollary~\ref{cor:optimal_alpha}, the optimal bounding parameter $\alpha^*$ depends on $T$, $\qc$, and $\qs$, and asymptotically converges to $\sqrt{\frac{2L\Gamma}{Z}}$ as $T\to\infty$.
Evaluating the components of $\psi_1$ under this asymptotic time limit reveals that both the quantization term and the data heterogeneity term scale with $\sqrt{Z}$. This relationship highlights a critical system dynamic when we subsequently extend this to the theoretical limit of quantization precision.
As $\qc$ and $\qs$ increase, 
the quantization error $Z$ decays exponentially
, meaning that the asymptotic optimality gap in \eqref{eq:optimal_gap} is driven closer 
to zero and becomes practically negligible in full-precision configurations ($\qmax$).
Therefore, under purely unquantized local training, the system converges to the exact global optimum regardless of the degree of data heterogeneity $\Gamma$. 
Furthermore, when combined with unquantized communication, our bound~\eqref{eq:converge} becomes consistent with the standard convergence result of \cite{li2020flconvergence}. 
However, Eq.~\eqref{eq:optimal_gap} reveals that reduced-precision local training restricts this convergence, creating a permanent residual error floor that is fundamentally amplified by the non-IID setting.

\subsection{Solution Approach} \label{sec:solution}

With the convergence behavior of the GQ-FSL framework formally established, we can now explicitly evaluate the target accuracy constraint in~\eqref{eq:opt_acc}. By substituting the optimal bounding parameter $\alpha^*$ and the theoretical upper bound~\eqref{eq:converge} into the accuracy requirement, we establish a direct analytical relationship between the target accuracy $\epsilon$ and the framework parameters. Although this resulting expression becomes highly nonlinear with respect to $T$, rendering closed-form algebraic inversion intractable, the bound remains strictly monotonically non-increasing in $T$. This property guarantees a unique intersection with any feasible target accuracy $\epsilon$. Therefore, we can employ a discrete bisection search to efficiently determine the minimum number of global rounds $T$ required to satisfy the constraint. This step successfully transforms the implicit accuracy requirement into a tractable mathematical condition, allowing us to proceed to reformulate and solve the overarching optimization problem in~\eqref{eq:opt_problem}.


Solving the mixed-integer nonlinear optimization problem formulated in~\eqref{eq:opt_problem} is inherently challenging due to the non-convex nature of the objective function and the complex interdependencies between the quantization precisions and the splitting point. However, because quantization precision levels are finite (typically $\qmax=32$ bits), DNN architectures consist of limited number of layers, and both client participation and local iterations are capped by practical system limits, the search space for the decision variables $\{\qc, \qs, \qu, \lc, K, I\}$ is strictly discrete and bounded. Therefore, the optimal configuration can be determined offline prior to training through an 
exhaustive search over the feasible set.

To facilitate this search, we must accurately evaluate the expected total EC in the objective~\eqref{eq:opt_obj} and the expected per-round latency in the constraint~\eqref{eq:opt_lat}. In our system, the server's broadcast energy is dictated by the client with the worst channel condition, while the per-round latency is bottlenecked by the overall slowest participating client due to synchronous global aggregation. Because both metrics depend on worst-case variables within the active group, evaluating their expectations requires calculating the expected maximum of a randomly sampled subset. In the following lemma, we derive the analytical expression for this calculation.

\begin{lemma}[Expected Maximum of a Finite Sample] \label{lem:exp_max}
    Let $\mathcal{U} = \{u_1, u_2, \dots, u_N\}$ denote a fixed and finite set of $N$ scalar values. Let $\mathcal{S}$ be a random subset of size $K$ drawn uniformly without replacement from $\mathcal{U}$. Let $u_{(1)} \le u_{(2)} \le \dots \le u_{(N)}$ represent the ordered statistics of the entire population $\mathcal{U}$. The expected maximum value of the sampled subset, defined as $U$, is given by 
    \begin{equation}
        \expec{U} = \expec{\max_{k \in \mathcal{S}} \{u_k\}} = \sum_{k=K}^{N} u_{(k)} \frac{\binom{k-1}{K-1}}{\binom{N}{K}}. \label{eq:expected_max}
    \end{equation}
\end{lemma}
\begin{IEEEproof}
    See Appendix~\ref{app:lem_2_proof}.
\end{IEEEproof}

\subsubsection{Expected Total EC Analytical Expression}
Since the subset $\mathcal{N}_t$ of $K$ clients is sampled uniformly at random without replacement from the total $N$ clients, we can express the expected total EC in~\eqref{eq:opt_obj} as 
\begin{align}
     & T \bigg[\expec{\Eb(\lc)} + \frac{K}{N}\sum_{k=1}^N\bigg(\Eku(\qu,\lc) \nonumber \\
     & \hspace{5mm}+  I b \paren{E_k^\mathrm{C}(\qc,\qs,\lc) + E_k^\mathrm{T}(\qc,\lc)}\bigg)\bigg]. \label{eq:expected_energy}
\end{align}

To evaluate the expected broadcasting energy $\expec{\Eb(\lc)}$, 
let $\rb_k = B_b \log_2\big(1 + \frac{P_b^\mathrm{tx} h_k}{N_0 B_b}\big)$ denote the theoretical point-to-point downlink rate for a specific client $k$ based on its channel gain $h_k$. Then, the effective broadcasting rate can be equivalently expressed as $\rb = \min_{k \in \mathcal{N}_t} \curly{\rb_k}$. By taking the expectation over Eq.~\eqref{eq:energy_brod}, the expected broadcast energy $\expec{\Eb(\lc)}$ can be written as
\begin{equation}
    \expec{\frac{\dc\qmax P_b^\mathrm{tx}}{\min_{k \in \mathcal{N}_t} \curly{\rb_k}}} = \expec{\max_{k\in\mathcal{N}_t}\curly{\frac{1}{\rb_k}}} \dc\qmax P_b^\mathrm{tx}.
\end{equation}

Defining $m_k = 1/\rb_k$ and applying Lemma~\ref{lem:exp_max} we obtain 
\begin{equation}
    \expec{\Eb(\lc)} = \dc \qmax P_b^\mathrm{tx} \sum_{k=K}^{N} m_{(k)} \frac{\binom{k-1}{K-1}}{\binom{N}{K}}, \label{eq:expected_brod_energy}
\end{equation}
where $m_{(k)}$ represents the $k$-th order statistic of $\{m_k\}_{k=1}^N$. 
Substituting~\eqref{eq:expected_brod_energy} into~\eqref{eq:expected_energy} yields a fully tractable analytical expression for the expected total EC.

\subsubsection{Expected Per-Round Latency Analytical Expression}
Evaluating the expected per-round latency for the constraint in~\eqref{eq:opt_lat} follows a similar logic. From Eq.~\eqref{eq:latency_round}, the duration of a global round is dictated by the straggler client in the active subset. Defining $n_k = \tku(\qu,\lc) + I b \, \tau^\mathrm{GD}_k(\qc,\qs,\lc)$ and applying Lemma~\ref{lem:exp_max} over the set of client delays, we obtain the expected per-round latency $\expec{\tau(\qc,\qs,\qu,\lc)}$ as
\begin{equation}
    \expec{\max_{k \in \mathcal{N}_t} \curly{\tku(\qu,\lc) + I b \, \tau^\mathrm{GD}_k(\qc,\qs,\lc)}\!} \hspace{-1pt}= \hspace{-1pt}\sum_{k=K}^{N} \hspace{-1pt} n_{(k)} \frac{\binom{k-1}{K-1}}{\binom{N}{K}} \rule[-6mm]{0pt}{0pt}.
\end{equation}

\section{Simulation Results and Analysis} \label{sec:results}
In this section, we evaluate the system-level performance and tradeoffs of our proposed GQ-FSL\footnote{The simulation code is available at \url{https://github.com/RothIdan/gqfsl-green-quantized-fsl}}. 
We begin by detailing the experimental setup. We then present and analyze
quantitative results that demonstrate the effectiveness of our approach under varying constraints and network architectures.

\subsection{Experimental Setup}
\subsubsection{Network and System Configuration}

We simulate a wireless edge network consisting of $N=50$ clients uniformly deployed within a $500\times 500\ \mathrm{m}^2$ cell around a central edge server. For the channel gain $h_k$, we assume a path loss exponent of $4$. Transmission powers are set to $P_k^\mathrm{tx}=P_\mathrm{s}^\mathrm{tx}=0.1$~W, and $P_b^\mathrm{tx}=5$~W. To accommodate the maximum receiver bandwidth constraints of heterogeneous 5G reduced capability (RedCap) IoT devices \cite{3gpp_ts38306_rel17,ericsson2023redcap}, the BS allocates orthogonal subchannels of $B=4$~MHz per device for uplink and unicast downlink, while broadcasting over a primary bandwidth of $B_b=20$~MHz. We further assume $N_0=-174$ dBm/Hz and use $I_{\max}=5$ alongside $\qmax=32$~bits.
For the physical computing model, we consider heterogeneous capabilities across the edge devices. We implement a $3$-tier computational capability model,  defined by MAC units, main SRAM size, and clock frequency $(p_k, S_k, f_k)$ which are uniformly sampled at random from three profiles: 'LOW' ($16$, $0.5$~MB, $200$~MHz), 'MID' ($64$, $2$~MB, $500$~MHz), and 'HIGH' ($256$, $4$~MB, $1.2$~GHz). Furthermore, each dedicated server instance is provisioned at the 'HIGH' tier specification. For the precision-dependent energy model we use $A=3.7$~pJ, $a=1.25$, and $A_{\mathrm{d}}=150$ as in \cite{moons2017minimum_qnn}.

\subsubsection{Training Configuration and Parameterization}
We evaluate GQ-FSL across two distinct architectures to demonstrate baseline superiority and deep network scalability.

\textit{Baseline Validation (CNN on MNIST):} Following the quantized FL reference experimental setup in~\cite{kim2023greenfl}, we simulate an image classification task on the MNIST dataset, utilizing a softmax classifier to measure the FSL model performance. 
We adopt their specific convolutional neural network (CNN) structure, which comprises five convolutional layers and three fully connected layers. This results in $8$ candidate splitting points, with parameters $x_\mathrm{in}=784$, $d\approx4.07\times10^5$, $M\approx39.9\times10^6$, and $O\approx87.5\times10^3$. Here, $M$ accounts for the convolutions and fully connected layer operations, which represent the vast majority of the computational weight. The dataset is distributed across devices in a non-IID fashion by allocating labels from a Dirichlet distribution with parameter $0.1$, and the mini-batch size is set to $b=32$. To strictly align with the theoretical bounds established in \cite{kim2023greenfl}, we utilize their empirically derived landscape parameters for this specific architecture. Accordingly, we set $L = 0.097$, $\mu = 0.05$, $G = 0.25$, $\Gamma = 0.6$, and adopt $\sigma_k$ from their codebase.

\textit{Main Validation and Scalability Evaluation (ResNet-18 on CIFAR-10):} Additionally, we consider a more complex task using the CIFAR-10 dataset, where each device trains a deep QNN-based ResNet-18 architecture \cite{he2016resnet}. To accommodate the spatial resolution of CIFAR-10 images, the initial convolutional layer is modified to use a $3\times3$ kernel with a stride of $1$, and the subsequent max-pooling layer is removed. In this setting, we consider $10$ candidate splitting points as in \cite{lin2024efficient_psl}, yielding $x_\mathrm{in}=3072$, $d\approx11.2\times10^6$, $M\approx0.55\times10^9$, and $O\approx0.31\times10^6$. 
The complete splitting-specific parameters are detailed in Table~\ref{tab:resnet18_split}. 
To maintain a consistent theoretical baseline across both setups, we carry over the parameters ($L,\mu,G,\sigma_k, b\text{ and } \Gamma$) from the previous configuration. 
This standardization of the theoretical convergence landscape allows us to isolate how the physical scaling of the model, specifically computational load, memory footprint, and transmission payloads, impacts the framework's energy and latency.

\begin{table}[t]
\centering
\caption{ResNet-18 Split Parameters across Candidate Split Points: MACs, Activations, Weights, and Cut Layer Payload.}
\label{tab:resnet18_split}
\resizebox{\columnwidth}{!}{
\begin{tabular}{@{}cccccccc@{}}
\toprule
\multirow{2}{*}{$\lc$} & \multicolumn{3}{c}{\textbf{Client Submodel}} & \textbf{Trans.} & \multicolumn{3}{c}{\textbf{Server Submodel}} \\
\cmidrule(lr){2-4} \cmidrule(lr){5-5} \cmidrule(l){6-8}
& $\Mc$ & $\Oc$ & $\dc$ & $x_\text{cut}$ & $\Ms$ & $\Os$ & $\ds$ \\
& \scriptsize{($\times 10^6$)} & \scriptsize{($\times 10^3$)} & \scriptsize{($\times 10^3$)} & \scriptsize{($\times 10^3$)} & \scriptsize{($\times 10^6$)} & \scriptsize{($\times 10^3$)} & \scriptsize{($\times 10^6$)} \\
\midrule
1 & 1.77 & 65.5 & 1.9 & 65.5 & 553.66 & 245.8 & 11.17 \\
2 & 77.27 & 131.1 & 75.8 & 65.5 & 478.16 & 180.2 & 11.10 \\
3 & 152.76 & 196.6 & 149.8 & 65.5 & 402.67 & 114.7 & 11.02 \\
4 & 211.48 & 229.4 & 380.0 & 32.8 & 343.95 & 81.9 & 10.79 \\
5 & 286.98 & 262.1 & 675.4 & 32.8 & 268.45 & 49.2 & 10.50 \\
6 & 345.70 & 278.5 & 1594.4 & 16.4 & 209.73 & 32.8 & 9.58 \\
7 & 421.20 & 294.9 & 2775.1 & 16.4 & 134.23 & 16.4 & 8.40 \\
8 & 479.92 & 303.1 & 6448.2 & 8.2 & 75.51 & 8.2 & 4.73 \\
9 & 555.42 & 311.3 & 11168.8 & 8.2 & 0.01 & 0 & 0.01 \\
10 & 555.43 & 311.3 & 11174.0 & 0 & 0 & 0 & 0 \\
\bottomrule
\end{tabular}
}
\end{table}

\subsubsection{Benchmarks and Statistical Evaluation}
While the analytical optimization in Section~\ref{sec:problem_formulation} is formulated for a specific network instance, the empirical evaluations are averaged over multiple independent trials to account for diverse physical conditions. Specifically, the channel gains and the devices' compute capabilities are varied to ensure the reported performance metrics robustly reflect expected system behavior under heterogeneous real-world conditions. 

To demonstrate the efficacy of our proposed joint optimization framework, we compare GQ-FSL against two primary benchmarks. The first is quantized FL (Q-FL) \cite{kim2023greenfl}, which represents the traditional non-split approach where the entire quantized model is processed locally ($\lc=\ell$). Hence, we only optimize $(\qc,K,I)$. The second is full-precision FSL (FP-FSL), a split version of~\cite{luo2021cost_effective_fl}, which results in optimization of $(\lc, K,I)$ while the parameters are represented in maximum bit-widths ($\qc=\qs=\qu=\qmax$).

\subsection{Performance Evaluation}


Before detailing the numerical results, we must contextualize the scope of this evaluation. These simulations probe system bottlenecks and physical scaling behaviors rather than prescribe fixed hyperparameters. Because theoretical convergence bounds do not capture real-world generalization, mathematically extreme settings (e.g., minimal client participation or a single local iteration) may fail in practice. Edge deployments therefore require empirical tuning to balance system efficiency and data coverage.
Thus, the presented numerical analysis should be interpreted as  conceptual design guidelines highlighting the architecture's underlying physical tradeoffs.

\begin{figure*}[t]
    \captionsetup[subfigure]{font=small, labelformat=parens, textfont=normalfont, labelfont=normalfont}
    \centering
    \begin{subfigure}[b]{0.45\textwidth}
        \centering
        \includegraphics[width=\textwidth]{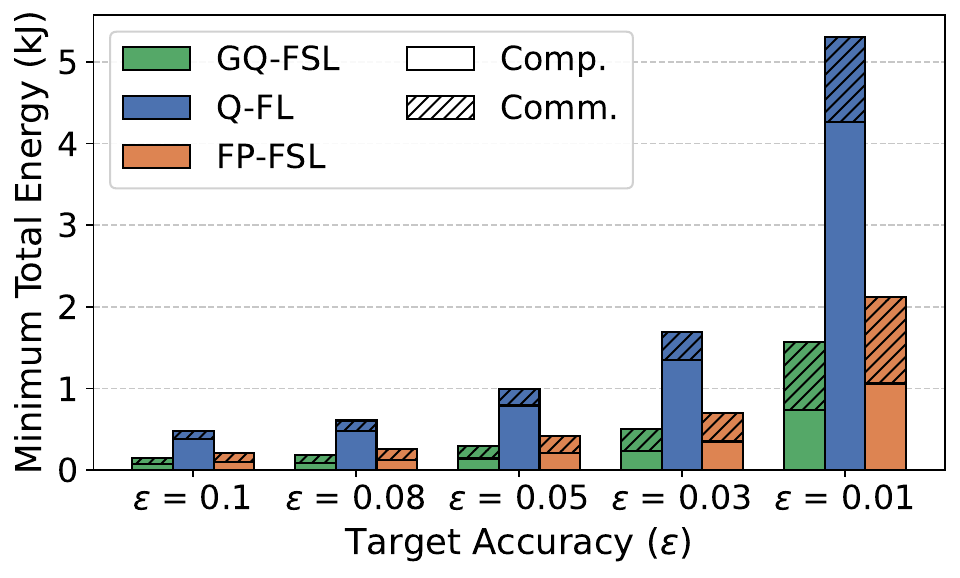}
        \caption{Baseline CNN on MNIST.}
        \label{subfig:mnist_energy}
    \end{subfigure}
    \hspace{6mm}
    \begin{subfigure}[b]{0.45\textwidth}
        \centering
        \includegraphics[width=\textwidth]{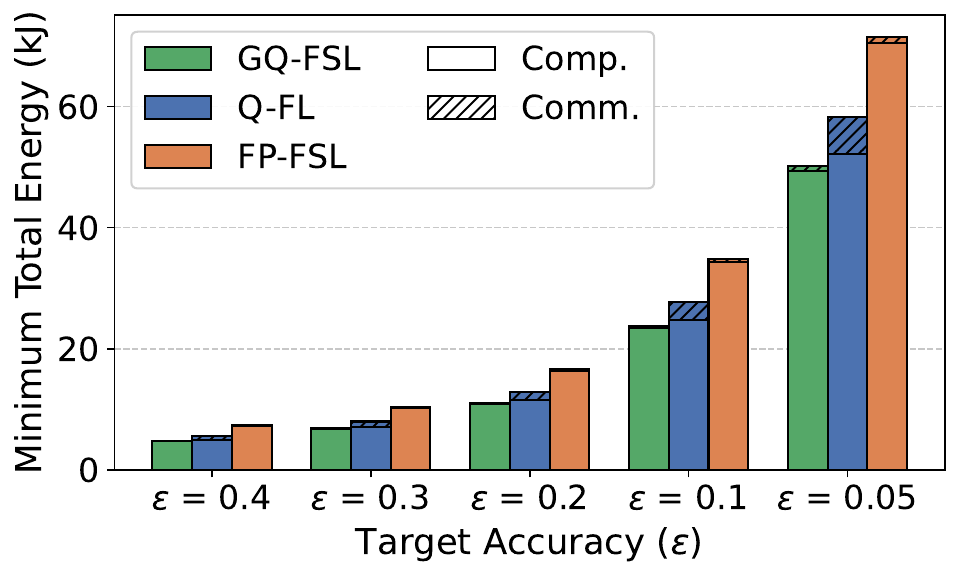}
        \caption{ResNet-18 on CIFAR-10.}
        \label{subfig:cifar_energy}
    \end{subfigure}
    \caption{Minimum total system energy consumption comparison across varying target accuracies ($\epsilon$) for an unconstrained latency setting ($\tau_{\max} \to \infty$) with fixed client participation $K = 10$. Energy is partitioned into computation and communication.}
    \label{fig:eng_acc_comp}
\end{figure*}

\begin{figure*}[t]
    \captionsetup[subfigure]{font=small, labelformat=parens, textfont=normalfont, labelfont=normalfont}
    \centering
    \begin{subfigure}[b]{0.45\textwidth}
        \centering
        \includegraphics[width=\textwidth]{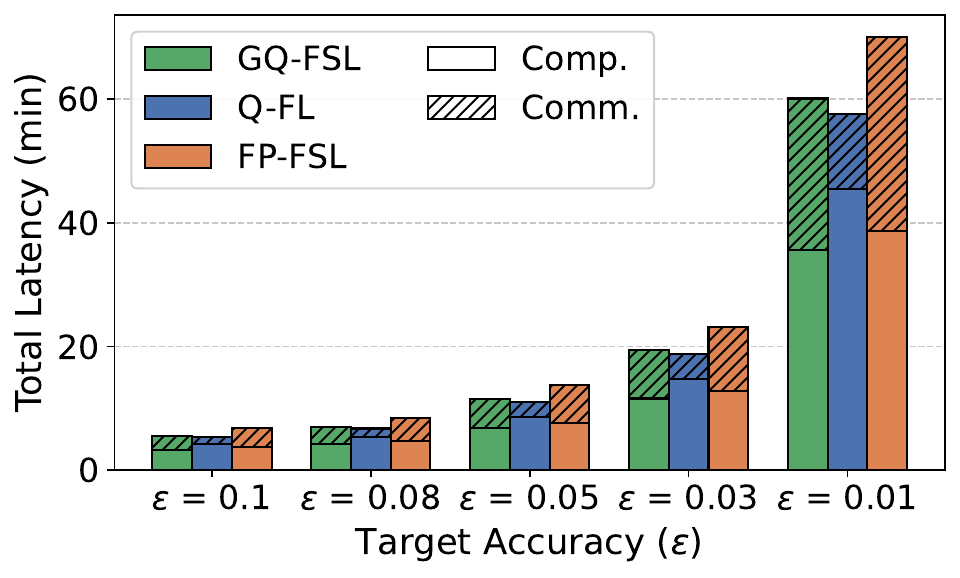}
        \caption{Baseline CNN on MNIST.}
        \label{subfig:mnist_latency}
    \end{subfigure}
    \hspace{6mm}
    \begin{subfigure}[b]{0.45\textwidth}
        \centering
        \includegraphics[width=\textwidth]{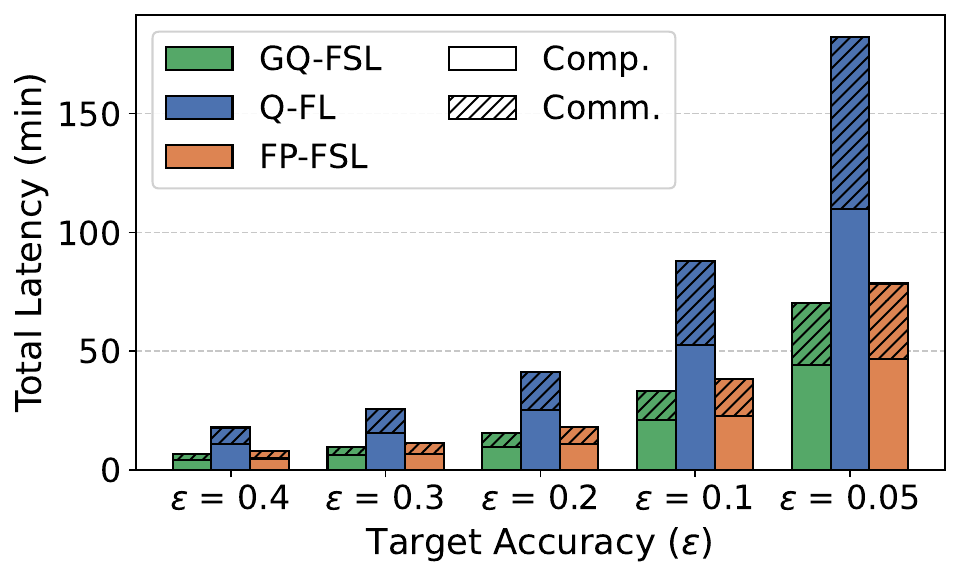}
        \caption{ResNet-18 on CIFAR-10.}
        \label{subfig:cifar_latency}
    \end{subfigure}
    \caption{Total training latency corresponding to the minimum energy configurations from Figure~\ref{fig:eng_acc_comp} across varying target accuracies ($\epsilon$) with fixed client participation ($K = 10$). Latency is partitioned into computation and communication.}
    \label{fig:lat_acc_comp}
\end{figure*}

\subsubsection{Baseline Superiority and the Energy Cost of Accuracy}

Fig.~\ref{fig:eng_acc_comp} shows the minimum total EC for GQ-FSL, Q-FL, and FP-FSL across varying target accuracy thresholds ($\epsilon$) for (a) the baseline CNN architecture on the MNIST dataset and (b) the ResNet-18 architecture on the CIFAR-10 dataset. We evaluate the frameworks under an unconstrained latency setting $(\tau_{\max} \to\infty)$ to identify the absolute minimum energy footprint of each configuration. Furthermore, the number of participating clients is fixed to $K=10$ to ensure a direct and fair comparison of structural efficiency. The results confirm GQ-FSL maintains strict energy superiority over both benchmarks at every accuracy level by leveraging its reduced-precision split architecture and computational server-side offloading. Naturally, the total EC increases as accuracy requirements tighten. At extreme target accuracy, the reduced quantization noise tolerance forces the frameworks to rely on extended global rounds and higher precision levels, drastically amplifying the total EC.

Comparing the two subplots reveals a vulnerability in standard Q-FL. For the shallow baseline CNN, Q-FL actually consumes significantly more energy than the FSL-based frameworks. Because the baseline CNN is over $96\%$ smaller than ResNet-18, the absolute energy saved per round via quantization is marginal. More importantly, fully local computation under Q-FL still exceeds the limited on-chip SRAM capacity of the lowest-tier edge devices. The FSL approach resolves this through computation workload balancing. By partitioning the the CNN at early layers ($\lc=2$), the resulting submodels fit entirely within the highly efficient SRAM of all clients and the edge server. Conversely, for massive architectures like ResNet-18, the total parameter count exceeds the SRAM capacity of even the edge server, making dominant off-chip DRAM access unavoidable. In this scaled regime, GQ-FSL pushes the split point deeper ($\lc=4$) to confine the majority of this massive data movement penalty to the server. This leaves the client-side submodels small enough to prevent excessive energy drain by better balancing the utilization of available SRAM (e.g., requiring only $0.9$~MB for weights and $0.5$~MB for activations at $\qc=19$ for $\epsilon=0.1$), and retaining only a lightweight fraction of the weights for transmission during global aggregation. Thus, whether by completely eliminating DRAM access in shallow networks or strategically offloading the DRAM burden in deep networks, GQ-FSL successfully establishes the absolute minimum energy footprint.



While Fig.~\ref{fig:eng_acc_comp} establishes the energy efficiency of the proposed framework, evaluating the practical time cost of these energy-optimal configurations is equally critical. Fig.~\ref{fig:lat_acc_comp} shows the total training latency, revealing a profound scalability advantage for GQ-FSL. For the shallow baseline CNN, standard Q-FL achieves a marginally faster training time. A time-domain breakdown explains this inversion: although server offloading successfully reduces local computation time, the energy-optimal split ($\lc=2$) requires transmitting a large intermediate activation payload ($x_\text{cut}=6400$). Although they are quantized in the forward pass, the returning backward gradients demand full-precision transmission. For this compact architecture, this asymmetrical communication overhead eclipses the computational time savings. 
However, this communication-computation tradeoff fundamentally inverts when scaling to deep networks. For the ResNet-18 architecture, the massive local computation requirements and the sheer volume of the full-model wireless payload completely overwhelm the standard Q-FL approach. By structurally partitioning the architecture ($\lc=4$) and offloading the computational bulk to the capable edge server, GQ-FSL simultaneously cuts the physical computation time by $60\%$ and the communication delay by $64\%$ at $\epsilon=0.05$. Consequently, even when strictly optimizing for minimum energy, GQ-FSL inherently yields superior practical training times for deep networks, proving that reduced-precision split learning is essential for scalable edge deployments.



\subsubsection{Energy-Latency Tradeoff and Optimal Parameter Analysis} \label{sec:pareto_analysis}

\begin{figure}[t]
    \centering
    \includegraphics[width=0.45\textwidth]{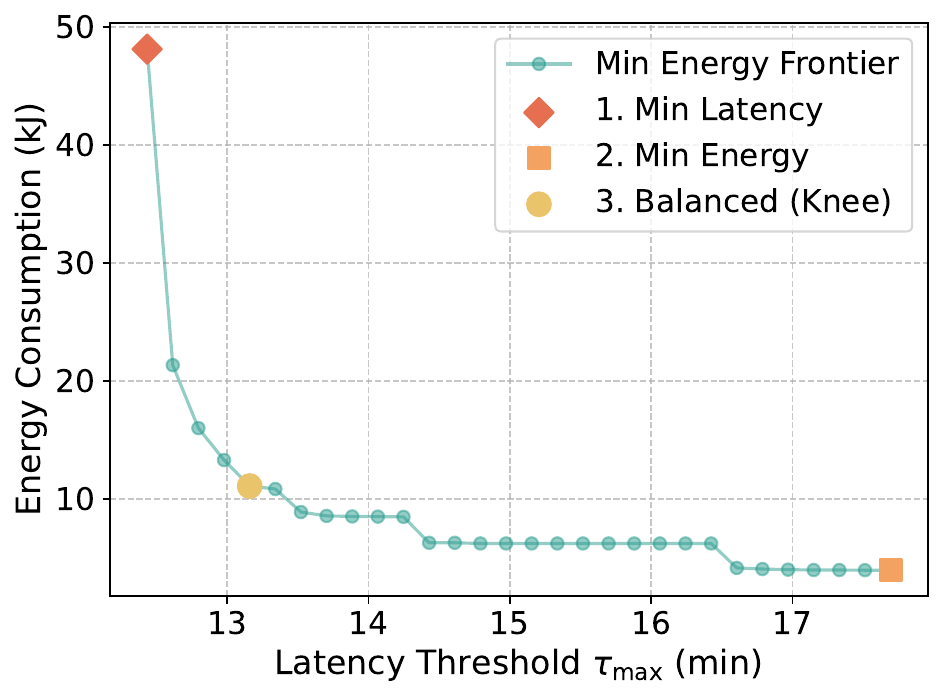}
    \caption{Minimum total system energy consumption across varying total training latency threshold ($\tau_{\max}$) for the ResNet-18 architecture on target accuracy $\epsilon = 0.1$.}
    \label{fig:eng_lat_pareto}
\end{figure}

For the remainder of our evaluation, we focus exclusively on the ResNet-18 architecture to provide a more practical and rigorous representation of modern edge AI workloads. Figure~\ref{fig:eng_lat_pareto} shows the optimal energy-latency Pareto frontier for GQ-FSL under a target accuracy constraint of $\epsilon=0.1$. The curve exposes a fundamental physical tradeoff between training latency and total EC, highlighted by the extreme energy penalty required to operate at the absolute minimum latency limit. To comprehensively understand the system dynamics driving this tradeoff, we analyze the optimal control variables at three critical points along the frontier: the minimum latency point, the minimum energy point, and a balanced knee point. The corresponding optimal configurations are detailed in Table~\ref{tab:optimal_points}. Before examining the distinct operating states, a structural consistency stands out across all three points: the optimizer selects the earliest possible split point ($\lc=1$). By offloading the maximum allowable parameter space to the capable edge server, the framework restricts the client workload to the absolute minimum. While a more balanced client-server hardware allocation can shift the optimal split point deeper into the architecture (e.g., $\lc=4$ for the minimum energy point under symmetric compute capacities \cite{gqfsl_spawc2026}), the aggressive offloading observed here directly neutralizes the hardware bottlenecks of low-tier devices, making GQ-FSL an exceptional fit for practical and highly constrained edge deployments.

\begin{table}[t]
\centering
\caption{Optimal Parameters and Performance Metrics Across the Pareto Frontier in Figure~\ref{fig:eng_lat_pareto}.}
\label{tab:optimal_points}
\resizebox{\columnwidth}{!}{%
\begin{tabular}{c c c c c}
\toprule
\textbf{} & \textbf{Min Latency} & \textbf{Balanced} & \textbf{Min Energy} & \textbf{Min Rounds} \\
\midrule
\multicolumn{5}{c}{\textit{Optimal Control Variables}} \\
\midrule
$\lc$ & 1 & 1 & 1 & 5 \\
$\qc$ & 13 & 14 & 16 & 28 \\
$\qs$ & 22 & 21 & 19 & 29 \\
\vspace{0.2mm}$\qu$ & 9 & 9 & 11 & 16 \\
$K$ & 19 & 4 & 1 & 50 \\
$I$ & 1 & 1 & 1 & 1 \\
\midrule
\multicolumn{5}{c}{\textit{Convergence and Performance Outcomes}} \\
\midrule
$T$ & 281 & 317 & 479 & 264 \\
\midrule
\textbf{Energy (kJ)} & \textbf{48.12} & \textbf{11.10} & \textbf{3.97} & \textbf{146.69} \\
\hspace{-1mm}\text{\raisebox{0.5ex}{$\llcorner$}} Comp. & 47.09 & 10.85 & 3.87 & 144.60 \\
\text{\raisebox{0.5ex}{$\llcorner$}} Comm. & 1.03 & 0.25 & 0.10 & 2.10 \\
\midrule
\textbf{Latency (s)} & \textbf{746.1} & \textbf{789.6} & \textbf{1061.6} & \textbf{2705.6} \\
\hspace{-1mm}\text{\raisebox{0.5ex}{$\llcorner$}} Comp. & 60.7 & 62.4 & 79.3 & 1800.9 \\
\text{\raisebox{0.5ex}{$\llcorner$}} Comm. & 685.3 & 727.1 & 982.4 & 904.7 \\
\bottomrule
\end{tabular}%
}
\end{table}


At the minimum latency point (Point 1), the optimizer drives rapid training by selecting a moderately sized client pool ($K=19$). This participation level ensures sufficient data diversity to keep the required global rounds relatively low. Additionally, it strategically reduces the inclusion risk of the weakest devices, thereby preventing straggler-induced communication delays. By relaxing the time constraint by less than a minute, the balanced knee point (Point 3) orchestrates a massive $77\%$ reduction in total EC. The optimizer achieves this primarily by shrinking client participation to $K=4$ and shifting one bit of precision from the server ($\qs=21$) to the clients ($\qc=14$). Because the vast majority of the computational load resides at the server, reducing the server precision by a single bit yields tremendous energy savings. Simultaneously, the framework cheaply increases the precision of the lightweight client-side submodel to stabilize the convergence rate. While dropping clients naturally increases the required global rounds, the massive reduction in parallel computation and active transmission energy completely overshadows the minor penalty of the extra training time.

Finally, at the minimum energy extreme (Point 2), the loose latency constraint allows the optimizer to drop participation to a single client ($K=1$), serializing the workload and maximizing the global rounds. This collapse into a sequential SL process perfectly aligns with the analytical expected energy objective formulated in~\eqref{eq:expected_energy}, which is monotonically increasing with respect to $K$. Specifically, the right-hand side is linear in $K$, and the server's expected broadcasting energy is dictated by the worst-link bottleneck of the selected clients. Sampling a larger pool statistically increases the probability of incorporating a device with poor channel conditions, inherently inflating the required transmission energy. Furthermore, beyond minimizing participation, this relaxed latency constraint enables a major rebalancing of the precision levels. The optimizer drops the server precision down to $\qs=19$ while raising the client precision to $\qc=16$. 

The strategic precision exchange observed as the latency constraint loosens reflects a fundamental shift in the dominant system bottleneck, validating the theoretical properties established in Theorem~\ref{thm:convergence}. 
Under strict time constraints, Table~\ref{tab:optimal_points} shows that wireless communication is the primary limiting factor. Thus, the optimizer minimizes client ($\qc$) and model-update ($\qu$) precisions to accelerate transmission while keeping server precision ($\qs$) high to control the convergence error floor. In contrast, as relaxed latency shifts the bottleneck to server-side computation, the optimizer shrinks $\qs$ to extract massive energy savings while compensating for the increased quantization noise by raising $\qc$ and $\qu$ at a minimal cost. Notably, a uniform precision restriction ($\qc=\qs$) could not achieve these energy savings as it fails to effectively navigate the energy-latency tradeoff under a single global bit-width. By decoupling device-level hardware constraints from server computational capabilities, GQ-FSL's extra degree of freedom grants the optimizer greater flexibility to achieve optimal states unreachable under uniform constraints.

\subsubsection{Training Latency vs. Convergence}
Prior quantized FL frameworks, such as \cite{kim2023greenfl}, fundamentally evaluate system efficiency as a tradeoff between EC and the number of global rounds $T$. Building upon that energy-convergence tradeoff logic, minimizing convergence requires high client participation and high precision levels (as exemplified in Theorem~\ref{thm:convergence}), which in turn maximizes EC and vice versa. However, relying purely on global rounds obscures the true physical bottlenecks of practical edge deployments. To empirically demonstrate the critical distinction between mathematical convergence and physical training latency, we isolated the absolute minimum global rounds required for convergence, denoted as Min Rounds in Table~\ref{tab:optimal_points}.

Achieving this mathematical minimum of $T=264$ indeed forces all $K=N$ clients to participate using near-maximum compute precision levels ($\qc=28, \qs=29$). Under conventional evaluation, this state might appear optimal for training speed. Yet, translating this to physical latency exposes a severe bottleneck: the minimum global rounds configuration takes over $45$ minutes ($2705.6$ seconds) to complete. In contrast, the minimum latency configuration requires $17$ additional global rounds ($T=281$) but finishes in only $12.4$ minutes.

This disparity reveals a counterintuitive dynamic: a theoretical deceleration in convergence does not inherently translate to a degradation in practical training latency. By intentionally dropping participation ($K=19$) and reducing precision levels ($\qc=13,\qs=22,\qu=9$), the minimum latency configuration degrades the mathematical convergence rate to achieve practical training acceleration. While reducing precision and participation inevitably increases $T$, it simultaneously shrinks both local computation time and wireless transmission payloads. This dual reduction compresses the physical execution time of every single round. Furthermore, attempting to minimize $T$ by simply increasing $K$ yields diminishing returns. Because round synchronization is strictly bottlenecked by the slowest device, pushing $K$ higher inevitably triggers the straggler effect and compounds delays across the limited computation capabilities and wireless bandwidth. 

Notably, mapping this complex relationship exposes a fundamental advantage: the optimal control parameters for minimizing energy and minimizing latency are closely aligned. Reducing precision and limiting client participation inherently drives down both the total energy footprint and the per-round execution time as opposed to convergence. This alignment allows the system to identify highly efficient configurations that strike a superior balance between the two metrics, perfectly exemplified by the aforementioned balanced knee point. This proves the absolute necessity of energy-latency tradeoff analysis of joint optimization for practical edge networks.

\section{Conclusion} \label{sec:conclusion}
This paper proposed the GQ-FSL framework to advance green edge AI and overcome the resource bottlenecks of deploying large-scale DNNs at the wireless edge. By integrating FSL and stochastic quantization with distinct precision levels for server- and client-side submodel computations, as well as the uplink communication of model updates, our framework successfully decouples strict client-level hardware constraints from global model accuracy. We formulated a joint optimization problem rooted in a rigorous theoretical convergence bound to strategically configure the network splitting point, precision levels, client participation, and local iterations. By systematically analyzing this optimization, we exposed a fundamental energy-latency tradeoff and established critical guidelines for practical edge deployments. 
Extensive simulations validate our theoretical analysis, confirming that GQ-FSL achieves superior energy efficiency and practical latency management compared to standard quantized FL and full-precision FSL architectures. Ultimately, this approach provides a highly sustainable framework for edge intelligence, enabling resource-constrained devices to actively participate in complex collaborative training. Possible directions for future work include exploring asynchronous model aggregation to further alleviate straggler bottlenecks, as well as extending the analytical framework beyond static and uniform settings to support online dynamic adaptation under time-varying channels alongside client-specific configurations tailored to individual device capabilities.


\begin{appendices}

\section{Proof of Theorem~\ref{thm:convergence}} \label{app:thm_1_proof}

\subsection{Additional Notations}
Following standard practice in literature \cite{li2020flconvergence, zheng2020design, kim2023greenfl}, we redefine $t$, throughout the remainder of this proof as the local iteration index, with a slight abuse of notation. Thus, $\wkt$ denotes the model parameters of the $k$-th client at local iteration $t$. We introduce an auxiliary variable $\vkto$ to represent the intermediate local model state resulting from a single mini-batch SGD step on $\wkt$, given by 
\begin{equation}
    \vkto= \wkt - \eta_t \nabla F_k(\bs{\tilde{w}}_{k}^{t}, \xi_k^{t}), \label{eq:aux_variable}
\end{equation}
where $ \bs{\tilde{w}}_{k}^{t} = Q\big(\wkt\big) = 
        \begin{bmatrix}
            Q_\mathrm{c}(\wkc^t) \vspace{0.2mm} \\
            Q_\mathrm{s}(\wks^t)
        \end{bmatrix}$, $\vkto= 
        \begin{bmatrix}
            \bs{v}_{k,\mathrm{c}}^{t+1} \vspace{0.2mm}\\
            \bs{v}_{k,\mathrm{s}}^{t+1}
        \end{bmatrix}$.
    
Since each client performs $I$ local iterations per global round, the model strictly follows the local update $\wkto=\vkto$ for all $t+1\notin\mathcal{I}$, where $\mathcal{I}=\{jI \ |\ j=1,2,\dotsc\}$ is the set of global aggregation steps. Conversely, when $t+1\in\mathcal{I}$, each client transmits the quantized client-side submodel update $\bs{\tilde{d}}_{k,\mathrm{c}}^{\raisebox{-0.6ex}{$\scriptstyle \, t+1$}}=Q_\mathrm{u}\big(\bs{d}_{k,\mathrm{c}}^{\,t+1}\big)$ for global aggregation. Only at these specific instances, the updated global model $\bs{w}^{t+1}$ is meaningful, and is explicitly formed as $\bs{w}^{t+1} = \bs{w}^{t+1-I} + \frac{1}{K}\sum_{k \in \mathcal{N}_{t+1-I}} \bs{\tilde{d}}_{k}^{\raisebox{-0.6ex}{$\scriptstyle \, t+1$}}$, where $\bs{w}^{t+1-I}$ is the previous global model generated at the BS. Combining these two conditions, the outcome of the $(t+1)$-th iteration is formally expressed as
\begin{equation}
    \hspace{2.5mm}\wkto = 
    \begin{cases}
        \vkto \hspace{38mm} \text{if} \quad t+1 \not\in \mathcal{I}, \\
        \displaystyle \bs{w}^{t+1-I} + \frac{1}{K}\sum_{k \in \mathcal{N}_{t+1-I}} \bs{\tilde{d}}_{k}^{\raisebox{-0.6ex}{$\scriptstyle \, t+1$}} \quad \text{if} \quad t+1 \in \mathcal{I}.
    \end{cases} \label{eq:model_update}
\end{equation} 
with
    \begin{equation*}
    \bs{\tilde{d}}_{k}^{\raisebox{-0.6ex}{$\scriptstyle \, t+1$}} = \begin{bmatrix}
        \bs{\tilde{d}}_{k,\mathrm{c}}^{\raisebox{-0.6ex}{$\scriptstyle \, t+1$}} \vspace{0.2mm} \\
        \bs{d}_{k,\mathrm{s}}^{\,t+1}
    \end{bmatrix} =
    \begin{bmatrix}
        Q_\mathrm{u}\big(\bs{v}_{k,\mathrm{c}}^{t+1} - \bs{w}_{\mathrm{c}}^{t+1-I}\big) \vspace{0.2mm} \\
        \bs{v}_{k,\mathrm{s}}^{t+1} - \bs{w}_{\mathrm{s}}^{t+1-I}
    \end{bmatrix}.
    \end{equation*}
Note that the server-side submodel update $\bs{d}_{k,\mathrm{s}}^{\,t+1}$ does not undergo quantization because it already resides on the server and is aggregated in full precision.

Finally, we define two additional auxiliary variables: $\bs{\bar{v}}^t = \frac{1}{N}\sum_{k=1}^N \vkt$, and $\bs{\bar{w}}^t = \frac{1}{N}\sum_{k=1}^N \wkt$. For convenience, we also denote $\bs{\delta}^t=\frac{1}{N}\sum_{k=1}^N \nabla F_k(\bs{\tilde{w}}_{k}^{t}, \xi_k^{t})$, and $\bs{\bar{\delta}}^t = \frac{1}{N}\sum_{k=1}^N \nabla F_k(\bs{\tilde{w}}_{k}^{t})$.

\subsection{Result for One Local Iteration}  \label{app:local_iter}
To prove Theorem~\ref{thm:convergence}, we first establish a lemma that bounds the expected optimality gap after a single local training iteration, building upon the proof in~\cite{kim2023greenfl}.

\begin{lemma}\label{lem:local_iter}
    Under Assumption~\ref{asm:loss_sg}, we have
    \begin{align}
        \expec{\norm{\bs{\bar{v}}^{t+1}-\bs{w}^*}^2} & \leq (1-\mu\eta_t)\expec{\norm{\bs{\bar{w}}^t-\bs{w}^*}^2} \nonumber \\
        &+ \eta_t(\alpha-\mu)\left(\frac{\dc}{2^{2\qc}} + \frac{\ds}{2^{2\qs}} \right) + \eta_t\frac{2L\Gamma}{\alpha}\nonumber \\
        &+ \eta_t^2\left(\!2L\Gamma \!+\! \sum_{k=1}^N \!\frac{\sigma_k^2}{N^2}\! \right)\! + 4\alpha\eta_t^3(I-1)^2G^2. \label{eq:one_iter}
    \end{align}
\end{lemma}
\begin{IEEEproof}
    From~\eqref{eq:aux_variable}, we have $\bar{\bs{v}}^{t+1} = \bs{\bar{w}}^t-\eta_t\bs{\delta}^t$. Therefore, we can write
    \begin{align}
    \label{eq:lemma3_non_expect}
        &\norm{\bar{\bs{v}}^{t+1}-\bs{w}^*}^2 = \underbrace{\norm{\bs{\bar{w}}^t-\bs w^* -\eta_t\bs{\bar{\delta}}^t}^2}_{A_1} \nonumber \\ 
        & \hspace{12mm} +2\eta_t \underbrace{\iprod{\bs{\bar{w}}^t-\bs w^* -\eta_t\bs{\bar{\delta}}^t}{\bs{\bar{\delta}}^t-\bs\delta^t}}_{A_2} + \underbrace{\eta_t^2\norm{\bs\delta^t-\bs{\bar{ \delta}}^t}^2}_{A_3}.
    \end{align}
    When taking the expectation, the term $A_2$ evaluates to zero since $\expec{\bs\delta^t}=\bs{\bar{\delta}}^t$. 
    
    To upper bound $A_3$, we expand the squared norm and apply $\expec{\nabla F_k(\bs{\tilde{w}}_{k}^{t}, \xi_k^{t})}=\nabla F_k(\bs{\tilde{w}}_{k}^{t})$ alongside Assumption~\ref{asm:loss_sg}, adapting~\cite[Eq.~63]{kim2023greenfl} under uniform client sampling, which yields
    \begin{equation}
        \eta_t^2\expec{\norm{\bs\delta^t-\bs{\bar{\delta}}^t}^2} \leq \eta_t^2\sum_{k=1}^N \frac{\sigma_k^2}{N^2}. \label{eq:A3}
    \end{equation}
    
    Next, we upper bound $A_1 $. By expanding the norm, we get
    \begin{equation}
        A_1 = \norm{\bs{\bar{w}}^t-\bs w^*}^2 \underbrace{-2\eta_t\iprod{\bs{\bar{w}}^t-\bs w^*}{\bs{\bar{ \delta}^t}}}_{B_1}+\underbrace{\eta_t^2\big\lVert \bs{\bar{\delta}}^t\big\rVert^2}_{B_2}. \label{eq:A1}
    \end{equation}
    Applying Jensen's inequality and the L-smoothness property of the loss function, where $\norm{\nabla F_k(\bs{\tilde{w}}_{k}^{t})}^2\!\leq\hspace{-0.5pt}2L(F_k(\bs{\tilde{w}}_{k}^{t})-F_k^*)$, $B_2$ is upper bounded as \cite[Eq.~(56)]{kim2023greenfl} 
    \begin{equation}
        B_2 \leq 2L\eta_t^2\sum_{k=1}^N \frac{1}{N}\left( F_k\left(\bs{\tilde{w}}_{k}^{t}\right)-F_k^* \right). \label{eq:B2}
    \end{equation}
    To upper bound $B_1$, we expand it using the definitions of $\bs{\bar{\delta}}^t$, to write \cite[Eq.~(53)]{kim2023greenfl}
    \begin{align}
        B_1 = &\underbrace{-2\eta_t\sum_{k=1}^N \frac{1}{N} \iprod{\bs{\bar{w}}_t-\wktq}{\nabla F_k\left(\wktq\right)}}_{C_1} \, \nonumber \\ &\underbrace{-2\eta_t\sum_{k=1}^N \frac{1}{N} \iprod{\wktq-\bs w^*}{\nabla F_k\left(\wktq\right)}}_{C_2}
    \end{align}
    Using $\mu$-convexity of the loss function, we have \cite[Eq.~(55)]{kim2023greenfl}
    \begin{align}
        C_2 \leq &-2\eta_t\sum_{k=1}^N \frac{1}{N}\left(F_k\left(\wktq\right) -F_k(\bs w^*) \right) \nonumber \\
        &- \mu\eta_t\sum_{k=1}^N \frac{1}{N}\norm{\wktq-\bs w^*}^2. \label{eq:C2}
    \end{align}

    At this stage, we diverge from the proof in~\cite{kim2023greenfl} to upper bound $C_1$. By introducing a tunable parameter $\alpha>0$, we establish a framework that allows us to derive the tightest possible theoretical guarantee as discussed in Section~\ref{sec:convergence}. Applying the Cauchy-Schwarz and AM-GM inequalities, we have for $\alpha>0$:
    \begin{align}
        -\iprod{\hspace{-0.8pt}\bs{\bar{w}}^t\!-\wktq\hspace{-0.8pt}}{\!\nabla F_k\left(\wktq\right)\hspace{-1pt}} &\!\leq\! \sqrt{\alpha}\norm{\wktq-\bs{\bar{w}}^t}\frac{1}{\sqrt{\alpha}}\norm{\nabla F_k\left(\wktq\right)} \nonumber \\ 
        &\!\leq\! \frac{\alpha}{2}\hspace{-1pt}\norm{\wktq\!-\bs{\bar{w}}^t}^2 \!+\! \frac{1}{2\alpha}\hspace{-0.8pt}\norm{\nabla F_k\left(\wktq\right)\hspace{-0.8pt}}^2 . \label{eq:cs_amgm}
    \end{align}
    Using~\eqref{eq:cs_amgm} and the $L$-smoothness property, the upper bound of $C_1$ is given by
    \begin{equation}
        C_1 \leq \alpha\eta_t\sum_{k=1}^N \hspace{-1.1pt}\frac{1}{N}\norm{\wktq-\bs{\bar{w}}^t}^2 + 2L\frac{\eta_t}{\alpha}\sum_{k=1}^N \hspace{-1.1pt}\frac{1}{N} \left(F_k\left(\wktq\right) -F_k^* \right). \label{eq:C1}
    \end{equation}

    Substituting Eqs.~\eqref{eq:B2}, \eqref{eq:C2}, and \eqref{eq:C1} into Eq.~\eqref{eq:A1}, $A_1$ can be upper bounded by
    \begin{align}
        A_1 &\leq \norm{\bs{\bar{w}}^t-\bs w^*}^2 + \underbrace{\alpha\eta_t\sum_{k=1}^N \frac{1}{N}\norm{\wktq-\bs{\bar{w}}^t}^2}_{D_1} \nonumber \\
        &+ \underbrace{2L\eta_t\left(\frac{1}{\alpha}+ \eta_t\right)\sum_{k=1}^N \frac{1}{N}\left( F_k\left(\wktq\right)-F_k^* \right)}_{D_{2a}} \nonumber \\
        &\underbrace{-2\eta_t\hspace{-0.9pt}\sum_{k=1}^N \!\frac{1}{N}\hspace{-0.3pt}\left(F_k\left(\wktq\right) \!-\!F_k(\bs w^*) \hspace{-0.9pt}\right)}_{D_{2b}} \underbrace{-\mu\eta_t \hspace{-0.9pt}\sum_{k=1}^N \! \frac{1}{N}\hspace{-0.2pt}\norm{\wktq\hspace{-1pt}-\hspace{-1pt}\bs w^*\hspace{-0.9pt}}^2}_{D_3}\hspace{-0.8pt}. \label{eq:A1_ub}
    \end{align}
    To upper bound $D_2\hspace{-0.5pt}=\hspace{-0.5pt}D_{2a}+D_{2b}$, we enforce $ 2L\eta_t\left(\frac{1}{\alpha}+ \eta_t\right) \leq 2\eta_t$. Thus, for $\eta_t \leq \frac{\alpha-L}{\alpha L}$ with $\alpha > L$, we derive the bound as
    \begin{align}
        D_2 &\leq 2L\eta_t\left(\frac{1}{\alpha} + \eta_t\hspace{-1pt}\right) \sum_{k=1}^N \frac{1}{N}\left( F_k(\bs w^*) -F_k^* \right) \nonumber \\
        &= \frac{2L\eta_t \Gamma}{\alpha} + 2L\eta_t^2\Gamma \label{eq:D2},
    \end{align}
    where $\Gamma =  F(\bs{w}^*)-\frac{1}{N}\!\sum_{k=1}^NF_k^* $ is the degree of data heterogeneity.
    For $\eta_t\leq2\eta_{t+I}$, we next obtain the upper bound of $D_1$ as
    \begin{align}
         D_1 &= \alpha\eta_t\sum_{k=1}^N \frac{1}{N}\norm{\wktq-\wkt}^2 + \alpha\eta_t\sum_{k=1}^N \frac{1}{N}\norm{\wkt-\bs{\bar{w}}^t}^2 \nonumber \\
         &\leq \alpha\eta_t\sum_{k=1}^N \frac{1}{N}\norm{\wktq-\wkt}^2 + 4\alpha\eta_t^3(I-1)^2G^2, \label{eq:D1}
     \end{align}
     where the first equality holds since, under expectation, the inner product $\iprod{\wktq-\wkt}{\wkt-\bs{\bar{w}}^t}$ equals zero based on Lemma~\ref{lem:quatization}, and the final inequality follows from the result in~\cite[Lemma~3]{li2020flconvergence}. Similarly, by expanding the squared norm in $D_3$ under expectation and applying Lemma~\ref{lem:quatization}, followed by Jensen's inequality and the definition of $\bs{\bar{w}}^t$, we obtain 
     \begin{equation}
        D_3 \leq -\mu\eta_t\sum_{k=1}^N \frac{1}{N}\norm{\wktq-\wkt}^2 -\mu\eta_t\norm{\bs{\bar{w}}^t-\bs w^*}^2. \label{eq:D3}
    \end{equation}
    Finally, substituting Eqs.~\eqref{eq:D2}, \eqref{eq:D1} and \eqref{eq:D3} into Eq.~\eqref{eq:A1_ub} yields the upper bound for $A_1$. Incorporating this result along with the bound for $A_3$ from Eq.~\eqref{eq:A3} into the expected value of Eq.~\eqref{eq:lemma3_non_expect}, and subsequently bounding the expected quantization error $\expec{\norm{\wktq-\wkt}^2}$ using Eq.~\eqref{eq:fsl_quant_err}, yields the final result stated in Lemma~\ref{lem:local_iter}.
    
\end{IEEEproof}

\subsection{Proof Completion} \label{app:induction}
To complete the proof of Theorem~\ref{thm:convergence}, we follow~\cite{zheng2020design} and define additional auxiliary variables for convenience: $\bs{\bar{u}}^t=\sum_{k=1}^N\bs{u}_k^t$, with $\bs{u}_k^t$ as defined in~\cite[Eq.~(66)]{kim2023greenfl}. Because our primary objective is to evaluate the convergence of the global model, we now restrict our analysis to the case where $t+1\in\mathcal{I}$.
As in \cite[Eq.~(67)]{kim2023greenfl}, we write
\begin{align}
    \norm{\bs{\bar{w}}^{t+1}-\bs{w}^*}^2 &= \underbrace{\norm{\bs{\bar{w}}^{t+1}-\bs{\bar{u}}^{t+1}}^2}_{E_1} + \underbrace{\norm{\bs{\bar{u}}^{t+1}-\bs{w}^*}^2}_{E_2} \nonumber \\
    &+ \underbrace{2\iprod{\bs{\bar{w}}^{t+1}-\bs{\bar{u}}^{t+1}}{\bs{\bar{u}}^{t+1}-\bs{w}^*}}_{E_3}.
\end{align}
Following~\cite{kim2023greenfl} and adopting the results in~\cite[Lemma~7]{zheng2020design} with slight modification to our quantization scheme, we have
\begin{align}
    \expec{\bs{\bar{w}}^{t+1}} &= \bs{\bar{u}}^{t+1}, \label{eq:wbar_ubar} \\
    \expec{\norm{\bs{\bar{w}}^{t+1}-\bs{\bar{u}}^{t+1}}^2} &\le \frac{4\dc IG^2}{K2^{2\qu}}, \label{eq:E1_ub}
\end{align}
since only the client-side submodel update is being quantized for uplink transmission. Then, $E_3$ evaluates to zero under expectation from~\eqref{eq:wbar_ubar}, and $E_1$ is bounded using Eq.~\eqref{eq:E1_ub}. To bound $E_2$, we expand the squared norm and write
\begin{align}
    E_2 &= \underbrace{\norm{\bs{\bar{u}}^{t+1}-\bs{\bar{v}}^{t+1}}^2}_{F_1} + \underbrace{\norm{\bs{\bar{v}}^{t+1}-\bs{w}^*}^2}_{F_2} \nonumber \\
    & + \underbrace{2\iprod{\bs{\bar{u}}^{t+1}-\bs{\bar{v}}^{t+1}}{\bs{\bar{v}}^{t+1}-\bs{w}^*}}_{F_3}.
\end{align}
As in~\cite{kim2023greenfl}, we adopt the results in~\cite[Lemma~4]{zheng2020design}:
\begin{align}
    \expec{\bs{\bar{u}}^{t+1}} &= \bs{\bar{v}}^{t+1}, \label{eq:ubar_vbar} \\
    \expec{\norm{\bs{\bar{v}}^{t+1}-\bs{\bar{u}}^{t+1}}^2} &\le 4\eta_t^2I^2G^2\frac{(N-K)}{K(N-1)}. \label{eq:F1_ub}
\end{align}
Then, under expectation $F_3$ evaluates to zero from~\eqref{eq:ubar_vbar}, and $F_1$ and $F_2$ are bounded using Eq.~\eqref{eq:F1_ub} and Lemma~\ref{lem:local_iter}, respectively. 
Consequently, we have
\begin{align}
     \expec{\norm{\bs{\bar{w}}^{t+1}-\bs{w}^*}^2} &\leq (1-\mu\eta_t)\expec{\norm{\bs{\bar{w}}^t-\bs{w}^*}^2}  \nonumber \\ & + \eta_t^3\psi'_3 + \eta_t^2\psi'_2 + \eta_t\psi_1, \label{eq:induction}
\end{align}
where \vspace{-1mm}
    \begin{IEEEeqnarray}{rCl}
        \psi_1 &=& (\alpha - \mu)\left( \frac{\dc}{2^{2\qc}} + \frac{\ds}{2^{2\qs}} \right) + \frac{2L\Gamma}{\alpha}, \label{eq:psi1_proof}
        \\
        \psi'_2 &=&  2L\Gamma + \sum_{k=1}^N \frac{\sigma_k^2}{N^2} + \frac{4I^2G^2(N-K)}{K(N-1)} + \frac{4\dc IG^2}{K2^{2\qu}}, \hspace{1mm} \label{eq:psi2_tag}
        \\
        \psi'_3 &=& 4\alpha(I-1)^2G^2. \label{eq:psi3_tag}
    \end{IEEEeqnarray}

Following the proof methodology in~\cite{zheng2020design}, we define $\Delta^{\hspace{-0.3mm}t}=\expec{\norm{\bs{\bar{w}}^t-\bs{w}^*}^2}$. With a diminishing learning rate $\eta_t = \frac{\beta}{t+\gamma}$, where $\beta>\frac{1}{\mu}$ and $\gamma>0$, such that $\eta_t\le2\eta_{t+I}$ and $\eta_1\le \frac{\alpha-L}{\alpha L}$ with $\alpha>L$, we next prove by induction that $\Delta^{\hspace{-0.3mm}t}\leq\frac{v}{t+\gamma}+\frac{\psi_1}{\mu}$,
where \vspace{-1mm}
\begin{IEEEeqnarray}{rCl}
    v &=& \max\left\{ \frac{\beta^2\psi_2}{(\beta\mu-1)}, \ (\gamma+1)\Delta^{\hspace{-0.3mm}1} \right\}, \\
    \psi_2 &=& \psi'_2 + \frac{\beta}{\gamma+1}\psi'_3.
\end{IEEEeqnarray}
The definition of $v$ ensures that our claim holds for $t=1$. Assuming the conclusion holds for some $t\ge1$, it follows that
\begin{IEEEeqnarray*}{rCl}
     \Delta^{\hspace{-0.3mm}t+1} &\leq& (1-\mu\eta_t) \Delta^{\hspace{-0.3mm} t} +\eta_t^3\psi'_3 + \eta_t^2\psi'_2 + \eta_t\psi_1 
    \\
    &\leq& \paren{\!1-\frac{\beta\mu}{t+\gamma}}\!\hspace{-0.5pt}\paren{\!\frac{v}{t+\gamma}+\frac{\psi_1}{\mu}\!} \!+\hspace{-0.6pt}\frac{\beta^2}{(t+\gamma)^2}\psi_2 +\hspace{-1pt}\frac{\beta}{t+\gamma}\psi_1
    \\
    &=& \frac{t+\gamma-1}{(t+\gamma)^2}v + \frac{\psi_1}{\mu} + \sqb{\frac{\beta^2}{(t+\gamma)^2}\psi_2 -\frac{(\beta\mu-1)}{(t+\gamma)^2}v}
    \\
    &\leq& \frac{t+\gamma-1}{(t+\gamma)^2}v + \frac{\psi_1}{\mu} \leq \frac{v}{(t+\gamma+1)} + \frac{\psi_1}{\mu} ,
\end{IEEEeqnarray*}
where the final inequalities follow from $v\geq\frac{\beta^2\psi_2}{\beta\mu-1}$, and $\frac{t+\gamma-1}{(t+\gamma)^2} \leq \frac{1}{t+1+\gamma}$. 

To establish Theorem~\ref{thm:convergence}, we first apply $\max\{a,b\}\leq a+b\,$ to decouple the terms in $v$. Furthermore, by jointly applying the $\mu$-strong convexity and the bounded SG variance detailed in Assumption~\ref{asm:loss_sg}, it can be shown that $\expec{\norm{\bs{\bar{w}}^1-\bs{w}^*}^2} \le \frac{G^2}{\mu^2}$. Substituting this into the bound, we get
\begin{equation}
    \expec{\norm{\bs{\bar{w}}^t-\bs{w}^*}^2} \leq \frac{1}{t+\gamma}\sqb{\frac{\beta^2\psi_2}{(\beta\mu-1)} +\frac{(\gamma+1)G^2}{\mu^2}}+\frac{\psi_1}{\mu \rule[-1.2ex]{0pt}{0pt}}.
\end{equation}
Finally, we apply the L-smoothness property of the loss function, where $\expec{F(\bs{\bar{w}}^t)-F(\bs{w}^*)} \le \frac{L}{2}\expec{\norm{\bs{\bar{w}}^t-\bs{w}^*}^2}$. By setting $\beta=\frac{2}{\mu}$ and $\gamma=\max\big(8\frac{L}{\mu}-1,I\big)$ as in~\cite{zheng2020design}, and evaluating the bound at a global round $T$ equivalent to $TI$ local iterations, we obtain Theorem~\ref{thm:convergence}.


\section{Proof of Corollary~\ref{cor:optimal_alpha}} \label{app:cor_1_proof}
Let $h(\alpha)$ denote the right-hand side of Eq.~\eqref{eq:converge}. We seek to minimize $h(\alpha)$ subject to two constraints that were induced from the proof in Appendix~\ref{app:local_iter}: (i) $\alpha > L$, and (ii) $\eta_t =\frac{\beta}{t+\gamma} \leq \frac{\alpha-L}{\alpha L}$ for all $t \geq 1$. The second constraint is equivalent to evaluating it at $t=1$ and substituting $\beta = \frac{2}{\mu}$ which yields $\alpha \geq \frac{\mu L(\gamma+1)}{\mu(\gamma+1)-2L}$. 

By taking the first derivative of $h(\alpha)$ with respect to $\alpha$ and setting $h'(\alpha) = 0$, we find the unconstrained minimum $\alpha_0$. Using the substitution variables $X$, $Y$, and $Z$ defined in Corollary~\ref{cor:optimal_alpha}, this algebraic step directly yields the square root term presented in Eq.~\eqref{eq:alpha_star}. 

We then take the second derivative to verify convexity. Since $h''(\alpha) = \frac{2L^2\Gamma}{\alpha^3\mu} > 0$ strictly holds for all $\alpha > 0$, the function $h(\alpha)$ is strictly convex in our domain of interest. Therefore, the optimal parameter $\alpha^*$ is simply the unconstrained minimum projected onto the feasible domain bounds, which produces the exact expression in Eq.~\eqref{eq:alpha_star}.

\section{Proof of Lemma~\ref{lem:exp_max}} \label{app:lem_2_proof}
Let $u_{(1)} \le u_{(2)} \le \dots \le u_{(N)}$ denote the ordered statistics of the finite population. Under uniform sampling without replacement, the probability mass function for the sample maximum $U = \max_{k \in \mathcal{S}} \{u_k\}$ being the $k$-th order statistic readily follows from finite-population order statistics \cite[Sec.~3.7]{arnold2008first} as
\begin{equation}
    \mathbb{P}(U = u_{(k)}) = \frac{\binom{k-1}{K-1}}{\binom{N}{K}}.
\end{equation}
Computing the expectation over all possible values of $U$ yields $\expec{U} = \sum_{k=K}^{N}u_{(k)}\mathbb{P}(U = u_{(k)})$, where the summation starts at $k=K$ since at least $K-1$ smaller or equal elements must exist in the population to complete the sample of size $K$. This concludes the proof of Eq.~\eqref{eq:expected_max}.

 \end{appendices}

\bibliographystyle{ieeetr}
\bibliography{utils/references}

\end{document}